\documentclass{article}
\usepackage{microtype}
\usepackage{graphicx}
\usepackage{booktabs} 
\usepackage{hyperref}

\usepackage[accepted]{icml2026}

\usepackage{amsmath}
\usepackage{amssymb}
\usepackage{mathtools}
\usepackage{amsthm}

\definecolor{blue}{rgb}{0.2, 0.2, 0.9}
\definecolor{qdelta}{rgb}{0.91, 0.92, 0.97}
\definecolor{avg}{rgb}{0.88, 0.92, 0.99}
\definecolor{avgblue}{rgb}{0.93, 0.93, 0.935}
\definecolor{bluelightgray}{rgb}{0.4, 0.7, 0.8}
\definecolor{electriclime}{rgb}{0.8, 1.0, 0.0}
\definecolor{malachite}{rgb}{0.04, 0.85, 0.32}
\definecolor{darkred}{rgb}{0.55, 0.0, 0.0}
\definecolor{black}{rgb}{0.0, 0.0, 0}
\definecolor{darkblue}{rgb}{0.0, 0.0, 0.50}
\definecolor{darkgreen}{rgb}{0.0, 0.2, 0.13}
\definecolor{darkorchid}{rgb}{0.6, 0.2, 0.8}
\definecolor{neonskyblue}{rgb}{0.4,1,1}
\definecolor{neongreen}{rgb}{0.6,0,0.6}

\newcommand{\secondbest}[1]{\underline{#1}}
\usepackage[capitalize,noabbrev]{cleveref}

\usepackage{amsthm}
\usepackage{thmtools}
\usepackage{thm-restate}

\theoremstyle{plain}

\theoremstyle{definition}

\theoremstyle{remark}

\usepackage[textsize=tiny]{todonotes}

\usepackage{changepage} 
\usepackage{algorithm}
\usepackage{algorithmic}

\usepackage[utf8]{inputenc}
\usepackage{caption}
\usepackage{multirow}
\usepackage{subcaption}
\usepackage{pbox}
\usepackage{graphicx}
\usepackage{bm}
\usepackage{siunitx}
\usepackage{lipsum}
\usepackage{mathtools}
\usepackage[table]{xcolor}
\usepackage{adjustbox}
\usepackage{wrapfig}
\usepackage{stfloats}

\usepackage{array}
\usepackage{booktabs}
\usepackage{colortbl}
\newcolumntype{C}{>{\centering\arraybackslash}p{1.15cm}}

\usepackage{tabularx}
\usepackage{arydshln}
\usepackage{nicefrac}       
\usepackage{microtype}      
\usepackage{xcolor}         
\usepackage{paralist}
\usepackage[table]{xcolor}

\usepackage{newfloat}
\usepackage{listings}

\begin{document}

\twocolumn[
  \icmltitle{Q-Delta: Beyond Key–Value Associative State Evolution}

  \icmlsetsymbol{equal}{*}

  \begin{icmlauthorlist}
    \icmlauthor{Sumin Park}{kaist}
    \icmlauthor{Seojin Kim}{kaist}
    \icmlauthor{Noseong Park}{kaist}
  \end{icmlauthorlist}

  \icmlaffiliation{kaist}{Korea Advanced Institute of Science and Technology (KAIST), Daejeon, Republic of Korea}

  \icmlcorrespondingauthor{Noseong Park}{noseong@kaist.ac.kr}

  \icmlkeywords{Machine Learning, ICML}

  \vskip 0.3in
]

\printAffiliationsAndNotice{}  

\begin{abstract}
Linear attention reformulates sequence modeling as recurrent state evolution, enabling efficient linear-time inference. 
Under the key–value associative paradigm, existing approaches restrict the role of the query to the readout operation, decoupling it from state evolution. We show that query-conditioned state readout induces a structured value prediction over accumulated memory that complements key-based retrieval.
Based on this insight, we propose {Q-Delta}, a query-aware delta rule that integrates mixed key--query prediction errors into state evolution, enabling jointly corrective dynamics while preserving delta-rule efficiency. We establish stability guarantees for the resulting dynamics and derive a hardware-efficient chunkwise-parallel formulation with a custom Triton implementation. Empirical results demonstrate stable optimization, competitive throughput, and consistent improvements over strong baselines on language modeling and long-context retrieval tasks. Code is available at \url{https://github.com/psmiz/Q-Delta}.
\end{abstract}

\section{Introduction}

The Transformer architecture achieves strong sequence modeling performance with its softmax-based self-attention mechanism~\citep{vaswani2023attentionneed}, but incurs quadratic time and memory complexity with respect to sequence length. This limitation has motivated a line of work on linear Transformers, which replace softmax attention with kernelized or algebraically decomposable feature mappings $\phi(\cdot)$ that allow the attention computation to be algebraically reordered as $\phi(Q)\bigl(\phi(K)^\top V\bigr)$, enabling linear-time inference and training scalability~\citep{chevalier2018larnnlinearattentionrecurrent,wang2020linformerselfattentionlinearcomplexity, katharopoulos2020transformersrnnsfastautoregressive}.
This factorization enables an online realization in which the term $\phi(K)^\top V$ is maintained as an incrementally updated state,
$S_t = \sum_{i=1}^t v_i \phi(k_i)^\top$, revealing the attention as recurrent state evolution where information is written into a memory state by key–value outer products and retrieved via a query readout, $o_t = S_t \phi(q_t)$.

This perspective leads to a unifying interpretation of linear attention as querying an evolving key–value associative memory. 
Rather than modeling explicit pairwise interactions between tokens, linear attention emphasizes how information is incrementally written, stored, and retrieved from a shared memory structure through iterative state updates. Under this view, a range of prior works, including kernelized linear attention~\citep{kitaev2020reformerefficienttransformer, wang2020linformerselfattentionlinearcomplexity,sun2023retentivenetworksuccessortransformer,yang2024gatedlinearattentiontransformers, sun2024cacheoncedecoderdecoderarchitectures} and selective state space models (SSMs)~\citep{gu2020hipporecurrentmemoryoptimal,
smith2023simplifiedstatespacelayers,gu2024mambalineartimesequencemodeling,dao2024transformersssmsgeneralizedmodels}, can be viewed as linear RNN–style architectures that replace explicit attention maps with structured state evolution with recurrent update rules. 

Purely additive updates, however, lack mechanisms for adaptive memory modifications, {failing} to selectively revise or remove previously stored information. This results in increased key collisions and degraded retrieval accuracy as sequence length grows~\citep{pmlr-v139-schlag21a}. Delta-rule–based updates~\citep{liu2024longhornstatespacemodels,yang2025parallelizinglineartransformersdelta,yang2025gateddeltanetworksimproving} address this limitation by refining the state in response to retrieval error, the discrepancy between the observed value and the value retrieved by the current key. Longhorn~\citep{liu2024longhornstatespacemodels} reveals that this error-driven update reduces to an online regression step on the key–value prediction objective, enabling selective modification of the recurrent state while preserving linear-time recurrence~\citep{liu2024longhornstatespacemodels}. Recent linear transformers and SSMs adopt this perspective to reinterpret recurrent state evolution as amortized online learning, providing a principled foundation for improved in-context retrieval and memory control~\citep{brown2020languagemodelsfewshotlearners, olsson2022incontextlearninginductionheads}.

Despite these advances, existing linear RNN models share a common structural assumption: state evolution is governed primarily by key–value interactions, while the query is used only to read out the evolved state.
While this separation follows naturally from the original attention formulation, it implicitly assumes that query plays no informative role in shaping state dynamics. We question this conventional view of query as a passive readout mechanism by re-examining the role of query-based readout in recurrent state update process.
In this work, we show that querying the state yields a value prediction that reflects information stored across the accumulated memory trace, providing a distinct but complementary signal to key-based retrieval.

Motivated by this observation, we introduce \textbf{Q-Delta}, a query-aware delta rule that enables predictive state evolution by incorporating query-conditioned feedback directly into recurrent memory updates. 
Q-Delta jointly considers a key-retrieved value estimate $\hat{v}_t = S_{t-1} k_t$ and a query-conditioned value prediction $\hat{o}_t = S_{t-1} q_t$, and updates the memory using a mixed correction signal that couples these complementary value estimators. We show that the resulting dynamics remain stable and satisfy global geometric error contraction under mild empirical conditions, and we further derive a chunkwise-parallel formulation compatible with hardware-efficient training implemented in Triton kernel.

Our main contributions are summarized as follows:
\vspace{-0.5em}
\begin{itemize}
    \item We revisit the role of query readout in linear attention, showing that it induces a structured value prediction over accumulated memory.
    \item We propose {Q-Delta}, a query-aware delta rule that integrates mixed key--query prediction errors into state evolution, together with a hardware-efficient chunkwise-parallel Triton implementation.
    \item We establish a stability theory for Q-Delta, proving one-step contraction and global stability of the mixed key--query error under empirical alignment conditions.
    \item Empirically, Q-Delta consistently outperforms prior linear Transformers and SSMs baselines on language modeling and long-context retrieval tasks.
\end{itemize}

\begin{figure}[t]
    \centering
    \includegraphics[width=0.49\textwidth]{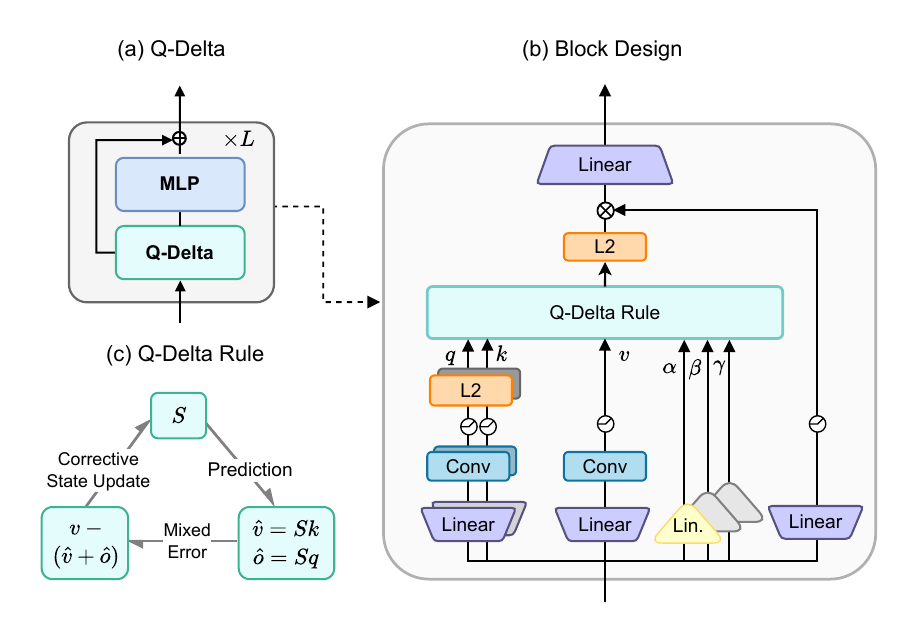}
    \caption{
    \textbf{Architecture overview and block design of Q-Delta.}
    (a) Q-Delta module within a Transformer block.
    (b) Block-level implementation illustrating how queries, keys, and values are projected and combined with gating signals.
    (c) The Q-Delta update rule, where the recurrent state produces a key-retrieved value $\hat{v}=Sk$ and a query-conditioned prediction $\hat{o}=Sq$, which are then combined into a mixed error for corrective state evolution.
    }
    \label{fig:architecture}
    \vspace{-1em}
\end{figure}

\vspace{-1em}
\section{Backgrounds}

\subsection{Linear Transformers}\label{sec:linear_transformers}

Recent work~\citep{katharopoulos2020transformersrnnsfastautoregressive, sun2023retentivenetworksuccessortransformer, yang2024gatedlinearattentiontransformers, liu2024longhornstatespacemodels, gu2024mambalineartimesequencemodeling, yang2025parallelizinglineartransformersdelta, yang2025gateddeltanetworksimproving} has shown that linear attention can be equivalently formulated as a linear recurrent model with a matrix-valued state. In its classic form, omitting normalization and feature activations, linear attention admits the recurrence
\begin{equation}
\label{eq:linear_attention}
S_t = S_{t-1} + v_t k_t^\top \in \mathbb{R}^{d_v \times d_k},
\quad
o_t = S_t q_t \in \mathbb{R}^{d_v},
\end{equation}
where $d_k$ and $d_v$ represent the (head) dimensions for $q_t, k_t \in \mathbb{R}^{d_k}$ and $v_t \in \mathbb{R}^{d_v}$ and $S_t$ accumulates rank one key--value outer products over time. Longhorn reframes this update rule as an online learning, interpreting the state update as the implicit solution of an online regression problem that learns a linear map from keys to values.~\cite{olsson2022incontextlearninginductionheads,liu2024longhornstatespacemodels} Under this view, designing a linear sequence-mixing model reduces to specifying an online loss and a regularizer that govern how new key--value information is incorporated into state~\citep{sun2025learninglearntesttime, yang2025parallelizinglineartransformersdelta, yang2024gatedlinearattentiontransformers, hu2025combaimprovingbilinearrnns}. This perspective provides a unified framework for understanding linear Transformers and their extensions as online linear regressors.

\paragraph{Delta-rule and gated extensions.}
While the additive update in Eq.~\eqref{eq:linear_attention} is efficient, it lacks a mechanism for selectively overwriting or correcting stored key-value associations. Delta-based models~\citep{liu2024longhornstatespacemodels, yang2025parallelizinglineartransformersdelta} address this limitation by modifying the state along the direction of the current key. The Deltanet~\cite{yang2025parallelizinglineartransformersdelta} updates the state as
\begin{equation}
\label{eq:delta_rule}
S_t = S_{t-1}(I - \beta_t k_t k_t^\top) + \beta_t v_t k_t^\top,
\end{equation}
where $\beta_t \in (0,1)$ controls the writing strength, dynamically erasing the old value $v_t^{\text{old}} = S_{t-1}k_t$ retrieved by $k_t$ and writing a new one $v_t^{\text{new}} = v_t$. GatedDeltaNet~\cite{yang2025gateddeltanetworksimproving} further augments this update with multiplicative gating, yielding recurrences in the form
\begin{equation}
\label{eq:gated_delta_recurrence}
S_t = S_{t-1}\bigl(\alpha_t(I - \beta_t k_t k_t^\top)\bigr) + \beta_t v_t k_t^\top,
\end{equation}
where $\alpha_t$ controls the state decay. Closely related decay-based formulations also arise in Mamba2~\citep{gu2024mambalineartimesequencemodeling, dao2024transformersssmsgeneralizedmodels}, whose SSM dynamics can be expressed as a linear recurrence with a decay term.

\paragraph{Explicit memory update via online regression.}
Beyond implicit memory encoded in recurrent state updates, a growing line of work treats memory as an explicit module that is continuously updated by online learning rules at inference time. Test-Time Training (TTT)~\citep{sun2025learninglearntesttime} optimizes the state via online gradient descent on a key-value prediction loss during both training and inference, 
\vspace{-0.5em}
\begin{equation}
S_t = S_{t-B} - \sum_{i=1}^B \eta_i \nabla_S \bigl\| S k_i - v_i \bigr\|^2.
\end{equation}
Similarly, Titans~\citep{behrouz2024titanslearningmemorizetest} introduce a neural long-term memory module whose parameters are updated at test time to memorize key--value associations, with decay and momentum controlling forgetting and retention.

\begin{table*}[t]
\centering
\scriptsize
\caption{Comparison of linear RNN models and their online learning objectives under the framework of~\citet{liu2024longhornstatespacemodels}.}
\label{tbl:online_objectives}
\renewcommand{\arraystretch}{1.8}
\setlength{\tabcolsep}{4pt}
\begin{tabular}{lll}
\toprule
\textbf{Method} & \textbf{Online Learning Objective} & \textbf{Online Update} \\
\midrule
LA
&
$\|S_t - S_{t-1}\|_F^2 - 2\langle S_t k_t,\, v_t\rangle$
&
$S_t = S_{t-1} + v_t k_t^\top$
\\
Mamba2
&
$\|S_t - \alpha_t S_{t-1}\|_F^2 - 2\langle S_t k_t,\, v_t\rangle$
&
$S_t = \alpha_t S_{t-1} + v_t k_t^\top$
\\
Longhorn
&
$\|S_t - S_{t-1}\|_F^2 - \beta_t \|S_t k_t - v_t\|_2^2$
&
$S_t = S_{t-1}\!\left(I - \epsilon_t k_t k_t^\top\right) + \epsilon_t v_t k_t^\top,\quad
\epsilon_t = \dfrac{\beta_t}{1+\beta_t\, k_t^\top k_t}$
\\
DeltaNet
&
$\|S_t - S_{t-1}\|_F^2 - 2\langle S_t k_t,\, \beta_t\bigl(v_t - S_{t-1}k_t\bigr)\rangle$
&
$S_t = S_{t-1}\!\left(I - \beta_t k_t k_t^\top\right) + \beta_t v_t k_t^\top$
\\
GatedDeltaNet
&
$\|S_t - \alpha_t S_{t-1}\|_F^2 - 2\langle S_t k_t,\, \beta_t\bigl(v_t - \alpha_t S_{t-1}k_t\bigr)\rangle$
&
$S_t = S_{t-1}\!\left(\alpha_t\left(I - \beta_t k_t k_t^\top\right)\right) + \beta_t v_t k_t^\top$
\\
\cellcolor{avgblue}
\textbf{Q-Delta (ours)}
&
\cellcolor{avgblue}
$\|S_t - \alpha_t S_{t-1}\|_F^2
- 2\Big\langle S_t k_t,\,
\beta_t\bigl(v_t - \alpha_t S_{t-1}k_t - \textcolor{blue}{\lambda_t\,\alpha_t S_{t-1}q_t}\bigr)
\Big\rangle$
&
\cellcolor{avgblue}
$S_t
=
S_{t-1}\!\left(\alpha_t\left(I - \beta_t\left(k_t k_t^\top + \textcolor{blue}{\lambda_t \, q_t k_t^\top}\right)\right)\right)
+
\beta_t v_t k_t^\top$
\\
\bottomrule
\end{tabular}
\end{table*}

\vspace{-0.5em}
\paragraph{Chunkwise Parallel Form.}
Although linear recurrences achieve an efficient linear-complexity with $\mathcal{O}(LD^2)$, their fully sequential nature limits training efficiency on modern hardware that favors parallelized computations. To address this, recent works reformulate linear recurrences in a {chunkwise parallel} manner, combining inter-chunk recurrence with intra-chunk parallel computation. The key idea is to partition the sequence into contiguous chunks of length $C$, allowing parallel computation within each chunk while maintaining a recurrent dependency across chunks. 
For the basic linear attention, the chunkwise formulation is
\begin{equation}
\begin{aligned}\label{eq:chunk_basic}
&S_{[t+1]} = S_{[t]} + V_{[t]}^\top K_{[t]}, \\
&O_{[t]} = Q_{[t]} S_{[t]}^\top + \bigl(Q_{[t]} K_{[t]}^\top \odot \mathrm{M}_{[t]}\bigr) V_{[t]},
\end{aligned}
\end{equation}
where $K_{[t]}, Q_{[t]}, V_{[t]} \in \mathbb{R}^{C \times D}$ stack the key, query, and value vectors within the chunk, and $\mathrm{M}_{[t]} \in \mathbb{R}^{C \times C}$ enforces causality within the chunk. 
More structured delta-rule recurrence can be expressed as 
\begin{equation}
\begin{aligned}
\label{eq:chunk_delta_state}
&S_{[t+1]}= S_{[t]} + \bigl(U_{[t]} - W_{[t]} S_{[t]}\bigr) K_{[t]}, \\
&O_{[t]}= Q_{[t]} S_{[t]}^\top 
+ \bigl(Q_{[t]} K_{[t]}^\top \odot M\bigr)
\bigl(U_{[t]} - W_{[t]} S_{[t]}\bigr).
\end{aligned}
\end{equation}
Here $U_{[t]}$ and $W_{[t]}$ are chunkwise matrices induced by the UT transform to ensure sequential delta update in chunk-level recurrence~\citep{joffrain2006accumulating,dominguez2018fast,yang2025parallelizinglineartransformersdelta}.
\begin{equation}
\begin{aligned}
&\mathbf{T}_{[t]}
=\left(
\mathbf{I}+\operatorname{tril}
\!\left(\operatorname{diag}(\boldsymbol{\beta}_{[t]})
\,\mathbf{K}_{[t]}\mathbf{K}_{[t]}^{\top},-1\right)
\right)^{-1}
\operatorname{diag}(\boldsymbol{\beta}_{[t]}), \\
&\mathbf{W}_{[t]} = \mathbf{T}_{[t]} \mathbf{K}_{[t]},
\quad 
\mathbf{U}_{[t]} = \mathbf{T}_{[t]} \mathbf{V}_{[t]}. 
\end{aligned}
\end{equation}
This formulation preserves the original delta-rule dynamics while enabling efficient hardware-parallelism.

\vspace{-0.5em}
\section{State Beyond Key–Value Association}

Across existing linear attention and SSMs, the state is predominantly interpreted as a key--value associative memory, while the query $q_t$ is used exclusively at readout time.
Under this interpretation, the query serves only as a passive readout mechanism and plays no role in shaping the state dynamics.
In this section, we question this assumption and show that the query readout itself encodes structured value information derived from the state, motivating a refined view of state evolution.

\vspace{-0.5em}
\subsection{Query for Value Prediction}

Prior work has shown that query-induced state readout, $o_t = S_tq_t$, following the linear recurrences can be expressed as a
value-weighted aggregation of past tokens~\citep{yang2025gateddeltanetworksimproving}.
We extend this characterization to
the readout taken from the prior state, $\hat{o}_t := S_{t-1}q_t$, and generalize it to an arbitrary linear transition
operator, so that the temporally mixed value-aggregation form holds uniformly across generic linear recurrence rules, including delta-rule and gated recurrences. We use this reformulation to motivate
query-conditioned state evolution in the following sections.

\paragraph{Query readout as temporally mixed value.}
Consider a generic form of recurrent state sequence $\{S_t\}_{t\ge1}$ defined as 
\begin{equation}
\label{eq:general_recurrence}
S_t \;=\; S_{t-1} P_t \;+\; \eta_t\, v_t k_t^\top,
\qquad
S_0 = 0,
\end{equation}
where $v_t \in \mathbb{R}^{d_v}$, $k_t \in \mathbb{R}^{d_k}$, $\eta_t \in \mathbb{R}$, and
$P_t \in \mathbb{R}^{d_k \times d_k}$ is a linear state transition operator.
Given this recurrence, state for each $t$ can be written as a linear combination of previously written value vectors
\begin{equation}
\label{eq:state_decomp}
S_{t-1} \;=\; \sum_{\tau=1}^{t-1} v_\tau\, b_{\tau,t-1}^\top,
\vspace{-0.5em}
\end{equation}
where the coefficient vectors $\{b_{\tau,t-1}\}_{\tau<t} \subset \mathbb{R}^{d_k}$ satisfy the backward recursion as
\begin{equation}
\label{eq:b_recursion}
b_{\tau,t} = P_t^\top b_{\tau,t-1}\quad(\tau < t),
\qquad
b_{t,t} = \eta_t k_t.
\end{equation}
Unrolling this gives the closed form, for any $\tau < t$,
\vspace{-0.5em}
\begin{equation}
\label{eq:b_closedform_gated}
b_{\tau,t-1}
=
\eta_\tau
\left(
\prod_{j=\tau+1}^{t-1} P_j^\top
\right)
k_\tau.
\vspace{-0.5em}
\end{equation}
Consequently, the query-conditioned prediction from the prior state, $S_{t-1} q_t$,
admits the temporally mixed value form as follows
\begin{equation}
\label{eq:pred_mix_gamma}
\hat{o}_t
= \sum_{\tau=1}^{t-1} \gamma_{\tau,t} v_\tau,
\qquad
\gamma_{\tau,t} := b_{\tau,t-1}^\top q_t \in \mathbb{R},
\end{equation}
so the query readout lies in the span of previously stored values and acts as a weighted value aggregation over past timesteps (see Appendix~\ref{app:qvp_proof} for derivations).
This result specializes to a standard linear attention ($P_t=I$)~\citep{katharopoulos2020transformersrnnsfastautoregressive},
the delta rule ($P_t = I-\beta_t k_t k_t^\top$)~\citep{yang2025parallelizinglineartransformersdelta},
and gated delta variants ($P_t=\alpha_t(I-\beta_t k_t k_t^\top)$)~\citep{yang2025gateddeltanetworksimproving}.

\paragraph{Attention over accumulated memory.}
Define the {time-evolved key} $\tilde{k}_{\tau, t}$ associated with past key--value pair $(k_\tau,v_\tau)$ at query time $t$ as
\begin{align}
\label{eq:time_evolved_key_gated}
\tilde{k}_{\tau,t}
:=& \left(
\prod_{j=\tau+1}^{t-1} P_j^\top
\right)
k_\tau \quad \in \mathbb{R}^{d_k},
\end{align}
which then gives $b_{\tau,t-1} = \eta_\tau \tilde{k}_{\tau, t}$.
Intuitively, $\tilde{k}_{\tau,t}$ encodes how the original key $k_\tau$ is transformed by subsequent state transitions via $\{P_j\}_{j=\tau+1}^{t-1}$.
It determines how strongly the current query $q_t$ can retrieve the value $v_\tau$ from the accumulated memory at current time $t$.

Using the definition of the time-evolved key $\tilde{k}_{\tau,t}$, the
query-induced prediction formed from the prior state,
$\hat{o}_t := S_{t-1} q_t$, can be written as
\begin{equation}
\label{eq:attention_over_memory}
\hat{o}_t
=
\sum_{\tau=1}^{t-1} \gamma_{\tau,t} v_\tau,
\qquad
\gamma_{\tau,t}
=
\eta_\tau\, q_t^\top \tilde{k}_{\tau,t}.
\end{equation}
Equivalently,
\vspace{-1.0em}
\begin{equation}
\label{eq:attention_form_expanded}
\hat{o}_t
=
\sum_{\tau=1}^{t-1}
\langle q_t, \tilde{k}_{\tau,t} \rangle_{\eta_\tau}
\, v_\tau .
\end{equation}
This has the form of unnormalized attention, in which the current query $q_t$
is matched against the time-evolved keys $\{\tilde{k}_{\tau,t}\}_{\tau<t}$ to mix
values stored across time.
Thus, $\hat{o}_t$ is a query-dependent value prediction obtained by attending
to the entire accumulated key--value memory, with the state transition operators
$\{P_j\}_{j=\tau+1}^{t-1}$ governing how past keys are reshaped over time.

\paragraph{Why query readout matters.}
The analysis above shows that the query-conditioned readout $\hat{o}_t = S_{t-1} q_t$
is a structured value aggregation driven by attending over the accumulated memory
with time-evolved keys.
This aggregation lies in the same value space as the key readout $S_{t-1}k_t$, but is weighted by attention-like similarities between the current query and past keys, rather
than key--key self-similarity.
As a result, $\hat{o}_t$ gives a query-induced  value information that is
already encoded in the state, which is not accessible through key-based recall alone.

This value-aggregation form is, on its own, an algebraic identity that holds for any
probe vector. What distinguishes the query is its role in the recurrence, $q_t$ is the
direction along which the state is finally read out, since the layer output is $o_t = S_t q_t$.
The query-conditioned prediction $\hat{o}_t = S_{t-1} q_t$ is not an arbitrary
projection of the state, but the model's own value prediction along the very direction
through which the memory is ultimately consumed downstream. Yet, conventional delta-rule corrects the state only against the key-retrieved
value $\hat{v}_t = S_{t-1}k_t$. The query readout $\hat{o}_t$ adds a complementary corrective
signal to state-evolution process, aligning it with the direction the state is actually read out, motivating its
inclusion in the update.

\subsection{Complementary Error Signals}

\paragraph{Mixed prediction errors.} 
Given keys $k_t \in \mathbb{R}^{d_k}$, values $v_t \in \mathbb{R}^{d_v}$, queries
$q_t \in \mathbb{R}^{d_k}$, and the prior state $S_{t-1} \in \mathbb{R}^{d_v \times d_k}$, we first compute two distinct value estimates,
\begin{equation}
\hat{v}_t := S_{t-1} k_t,
\qquad
\hat{o}_t := S_{t-1} q_t .
\end{equation}
Here, $\hat{v}_t$ corresponds to the {key-retrieved value}, 
while $\hat{o}_t$ is a {query-conditioned prediction}.

Intuitively, the two estimates $\hat{v}_t$ and $\hat{o}_t$ correspond to distinct projections of the same accumulated memory state.
The key-based readout $\hat{v}_t = S_{t-1} k_t$ evaluates the memory along the direction of the current key, yielding the value currently associated with $k_t$ under the learned key–value mapping.
In contrast, $\hat{o}_t = S_{t-1} q_t$ evaluates the same state along a different direction specified by the query, producing a value aggregation shaped by how the query aligns with time-evolved keys in memory.
Although both readouts are formed by aggregating past values, they generally induce different weightings over those values.
This distinction indicates that $\hat{v}_t$ and $\hat{o}_t$ encode complementary information present in the state, revealed by different projections.
Therefore, incorporating both in the state update process corrects this mixed prediction error, enabling more informative memory updates.

\begin{figure}[h]
    \centering
    \includegraphics[width=0.49\textwidth]{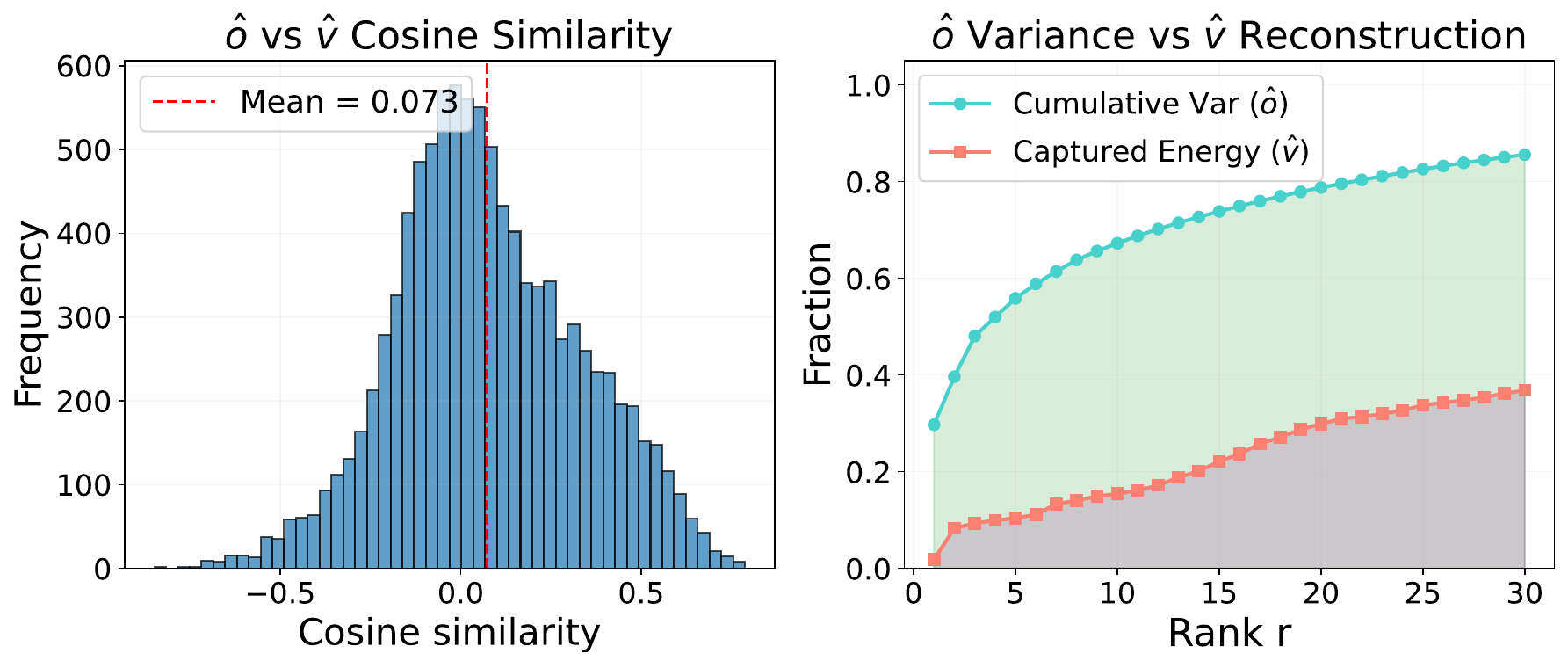}
    \caption{
    {Complementarity analysis between $\hat{v}$ and $\hat{o}$.}
    {Left:} Distribution of cosine similarity between $\hat{v}$ and $\hat{o}$, showing low alignment on average.
    {Right:} Cumulative $\hat{o}$ variance and $\hat{v}$ reconstruction energy across principal subspace rank $r$ of $\hat{o}$.
    }
    \label{fig:comp_vo}
    \vspace{-1em}
\end{figure}

Figure~\ref{fig:comp_vo} provides empirical evidence that the key-retrieved value
$\hat{v}_t = S_{t-1}k_t$ and the query-conditioned prediction
$\hat{o}_t = S_{t-1}q_t$ encode complementary, rather than redundant,
information from the recurrent state.
{Left:} we plot the distribution of cosine similarities
$\langle \hat{o}_t, \hat{v}_t \rangle / (\|\hat{o}_t\|\,\|\hat{v}_t\|)$
gathered across 10000 timesteps and 3 layers (5, 10, 15) of 340M Q-Delta.
The distribution is centered close to zero (mean $\approx 0.07$),
indicating that $\hat{o}_t$ and $\hat{v}_t$ occupy largely decorrelated
directions in value space, despite being derived from the same state
$S_{t-1}$.
{Right:} we analyze complementarity at the subspace level by
comparing the cumulative variance explained by the principal components of $\hat{o}$ with the fraction of $\hat{v}$ energy captured when projected onto the corresponding subspace of $\hat{o}$.
Specifically, we perform PCA on samples of $\hat{o}$ and measure the
reconstruction energy of $\hat{v}$ under the top-$r$ principal
subspace.
The substantial gap between the two curves shows that directions
accounting for most of the variance of $\hat{o}$ explain only a limited portion of the energy of $\hat{v}$.
Together, these results indicate that query-based predictions
emphasize value components that are not well represented by key-based
recall alone, supporting the use of mixed errors, $\hat{v}_t$ and $\hat{o}_t$, as complementary error signals for the state evolution under Q-Delta update rule.

\subsection{Q-Delta: Query-Aware Delta Rule}

We now introduce {Q-Delta}, a query-aware extension of the delta rule that incorporates
query-conditioned prediction feedback into state evolution.
Q-Delta builds upon delta-based associative memory updates, while allowing both key and query to participate in correcting the stored state. 
Table~\ref{tbl:online_objectives} summarizes Q-Delta in comparison to prior linear RNN models under a unified online learning objective framework and
\Cref{fig:architecture} illustrates the mechanism of Q-Delta.

\subsubsection{Sequential recurrence.}

\paragraph{Q-Delta rule}
We propose Q-Delta, a query-aware delta rule:
\begin{equation}
\label{eq:qdelta_state_update}
S_t
=
S_{t-1}
+
\beta_t \bigl(v_t - \hat{v}_t - \textcolor{blue}{\lambda_t \hat{o}_t} \bigr) k_t^\top ,
\end{equation}
where $\beta_t \in [0,1]$ controls the update strength and $\lambda_t \in [0,1]$ modulates the
influence of query-based feedback.

Rewriting Eq.~\eqref{eq:qdelta_state_update} yields an equivalent linear form,
\begin{equation}
\label{eq:qdelta_linear_form}
S_t
=
S_{t-1}
\bigl(I - \beta_t (k_t k_t^\top + \textcolor{blue}{\lambda_t q_t k_t^\top})\bigr)
+
\beta_t v_t k_t^\top .
\end{equation}
Including a forget gate $\alpha_t \in (0,1)$, the final Q-Delta update rule is as follows:
\begin{equation}
\begin{aligned}
\label{eq:qdelta_decay}
S_t &=
\alpha_t S_{t-1}
\bigl(I - \beta_t (k_t k_t^\top + \textcolor{blue}{\lambda_t q_t k_t^\top})\bigr)
+
\beta_t v_t k_t^\top . \\
o_t &= S_t q_t .
\end{aligned}
\end{equation}

\vspace{-1em}
\subsubsection{Chunkwise parallel form.}
We now derive a hardware-efficient chunkwise-parallel formulation for Q-Delta referring to the chunkwise expansion strategy of GatedDeltaNet.
Defining \textcolor{blue}{$x_t := k_t + \lambda_t q_t$}, the Q-Delta recurrence follows:
\begin{equation}
\label{eq:qdelta_gated_form}
S_t
=
\alpha_t S_{t-1}
\bigl(I - \beta_t \textcolor{blue}{x_t} k_t^\top\bigr)
+
\beta_t v_t k_t^\top,
\end{equation}
where $\lambda_t \in (0,1)$ is a learnable head-wise query-feedback coefficient
(see Appendix~\ref{app:lambimpl} for parameterization).
Fix a chunk indexed by $[t]$ consisting of $C$ consecutive timesteps
$\{t_1,\dots,t_C\}$, and denote the chunk entrance state by
$S_{[t]} := S_{[t]}^0 = S_{[t-1]}^C$.
For each timestep $t_i$, define
$P_{t_i} := I - \beta_{t_i} \textcolor{blue}{x_{t_i}} k_{t_i}^\top .$
By partially expanding the recurrence, the state after $r \le C$
steps within the same chunk can be written as
\vspace{-0.5em}
\begin{equation}
\label{eq:qdelta_chunk_expand}
S^r_{[t]}
=
S_{[t]}
\underbrace{\Bigl(\prod_{i=1}^r \alpha_{t_i} P_{t_i}\Bigr)}_{=:F^r_{[t]}}
+
\underbrace{
\sum_{i=1}^r
\Bigl(
\beta_{t_i} v_{t_i} k_{t_i}^\top
\prod_{j=i+1}^r \alpha_{t_j} P_{t_j}
\Bigr)
}_{=:G^r_{[t]}} .
\end{equation}
Let $\gamma^r_{[t]} := \prod_{i=1}^r \alpha_{t_i}$.
Then $F^r_{[t]} = \gamma^r_{[t]} P^r_{[t]}$, where
\begin{equation}
P^r_{[t]} := \prod_{i=1}^r P_{t_i} = \prod_{i=1}^r (I - \beta_{t_i} \textcolor{blue}{x_{t_i}} k_{t_i}^\top).
\end{equation}
Following the extended WY representation~\citep{Bischof1985TheWR} from GatedDeltaNet,
there exist vectors
$ w^i_{[t]} \in \mathbb{R}^{d_k}$ and
$ u^i_{[t]} \in \mathbb{R}^{d_v}$ defined as 
\begin{equation}
\begin{aligned}
w^r_{[t]}
&=
\beta_{t_r}
\left(\textcolor{blue}{x_{t_r}}
-
\sum_{i=1}^{r-1}
w^i_{[t]} \bigl(k_{t_i}^\top \textcolor{blue}{x_{t_r}}\bigr)
\right), \\
\tilde u^r_{[t]}
&=
\beta_{t_r}
\left( v_{t_r}
-
\sum_{i=1}^{r-1}
 \frac{\gamma_{[t]}^r}{\gamma_{[t]}^i} \tilde u^i_{[t]} \bigl(k_{t_i}^\top \textcolor{blue}{x_{t_r}}\bigr)
\right),
\vspace{-1em}
\end{aligned}
\end{equation}
such that (derivations in~\ref{app:wy_representation})
\vspace{-0.5em}
\begin{equation}
P^r_{[t]} = I - \sum_{i=1}^r  w^i_{[t]} k_{t_i}^\top,
\quad
G^r_{[t]} = \sum_{i=1}^r  \frac{\gamma_{[t]}^r}{\gamma_{[t]}^i} \tilde u^i_{[t]} k_{t_i}^\top .
\end{equation}
Substituting the WY forms into Eq.~\eqref{eq:qdelta_chunk_expand} gives
\begin{align}
S^r_{[t]}
&=
\gamma^r_{[t]} S_{[t]} P^r_{[t]} + G^r_{[t]} \nonumber \\
&=
\gamma^r_{[t]} S_{[t]}
\Bigl(I - \sum_{i=1}^r  w^i_{[t]} k_{t_i}^\top\Bigr)
+
\sum_{i=1}^r  \frac{\gamma_{[t]}^r}{\gamma_{[t]}^i} \tilde u^i_{[t]} k_{t_i}^\top \nonumber \\
&=
\gamma^r_{[t]} S_{[t]}
+
\sum_{i=1}^r
\bigl(
\tilde u^i_{[t]} - \gamma^i_{[t]} S_{[t]}  w^i_{[t]}
\bigr) \frac{\gamma_{[t]}^r}{\gamma_{[t]}^i}
k_{t_i}^\top .
\end{align}
At the end of the chunk ($r=C$), define the scaled state
$\overrightarrow S_{[t]} := \gamma^C_{[t]} S_{[t]}$,  
$ \overleftarrow W_{[t]} := [\gamma_{[t]}^1 w^1_{[t]},\dots, \gamma_{[t]}^C w^C_{[t]}]$, $Q_{[t]} = [q_{t_1}, \ldots, q_{t_C}]$, and $\overrightarrow K_{[t]} := (\Gamma_{[t]})_{C (\cdot)} \ K_{[t]}$ where $(\Gamma_{[t]})_{ij} = \frac{\gamma_{[t]}^i}{\gamma_{[t]}^j}$.
Then chunk-level state update admits the compact form
\begin{equation}
\begin{aligned}
&S_{[t+1]}
=
\overrightarrow S_{[t]}
+
\bigl(
 \widetilde U_{[t]} - \overleftarrow W_{[t]} S_{[t]}
\bigr)^\top
\overrightarrow K_{[t]}\\
&O_{[t]}
=
\, \overleftarrow Q_{[t]} S_{[t]}^\top
+
\bigl( Q_{[t]} K_{[t]}^\top \odot M\bigr)
\bigl( \widetilde U_{[t]} - \overleftarrow W_{[t]} S_{[t]}\bigr)
\end{aligned} 
\end{equation}
such that $S_{[t+1]} \in \mathbb{R}^{d_v \times d_k}$ and $O_{[t]} \in \mathbb{R}^{C \times d_v}$ and $M$ is the causal mask.
The vectors $\widetilde U_{[t]}$ and $W_{[t]}$ can be computed efficiently using a UT transform referring to DeltaNet, yielding a hardware-efficient chunkwise-parallel algorithm for Q-Delta (derivations in Appendix~\ref{app:ut_transform}):
\begin{equation}
\begin{aligned}
\label{eq:qdelta_chunk_U}
\widetilde U_{[t]}
=&
\Bigl[
I + \mathrm{Lower}\!\bigl(
\mathrm{d}(\beta_{[t]})
(\Gamma_{[t]} \odot \textcolor{blue}{X_{[t]}} K_{[t]}^\top)
\bigr)
\Bigr]^{-1}
\mathrm{d}(\beta_{[t]}) V_{[t]},\\
{W}_{[t]}
=&
\Bigl[I+\mathrm{Lower}\!\bigl(\mathrm{d}(\beta_{[t]})\textcolor{blue}{X_{[t]}}K_{[t]}^\top\bigr)\Bigr]^{-1}
\mathrm{d}(\beta_{[t]})\,\textcolor{blue}{X_{[t]}}.
\end{aligned}
\end{equation}
We also provide a Triton~\citep{tillet2019triton} kernel specific to both fully recurrent and chunkwise parallelized Q-Delta.

\subsubsection{Stability Analysis of Q-Delta Dynamics}

Here we analyze the prediction error dynamics of the proposed Q-Delta update rule.
While Q-Delta is motivated by correcting a mixed prediction error involving both key-
and query-induced memory readouts, its update rule does not correspond to a strict gradient
descent step on $\|v_t - S_{t-1}(k_t + \lambda q_t)\|^2$ unlike other delta rules under standard key-value association paradigm.
We therefore provide a theoretical analysis of the stability and error contraction
properties induced by this recurrence, showing that key--query jointly corrective feedback leads to controlled error dynamics under mild empirical conditions, despite not under a strict gradient descent interpretation.

\begin{restatable}[One-step contraction of mixed prediction error under Q-Delta]{lemma}{onesteplemma}
\label{lemma:onestep_contraction}
Let $k_t,q_t\in\mathbb{R}^d$ and $\lambda_t\in[0,1]$, and define the mixed input
$
x_t := k_t + \lambda_t q_t.
$
Consider the Q-Delta update
\[
S_t
=
S_{t-1}
+
\beta_t\bigl(v_t - S_{t-1}x_t\bigr)k_t^\top,
\qquad \beta_t \in (0,1].
\]
Assume that the scalar alignment
$
a_t := k_t^\top x_t
$
satisfies $\beta_t a_t \in (0,2)$ almost surely, and define
\[
\rho := \sup_t |1-\beta_t a_t| \in (0,1).
\vspace{-0.5em}
\]
Then the mixed prediction error contracts in one step:
\[
\|v_t - S_t x_t\|
\;\le\;
\rho\,\|v_t - S_{t-1}x_t\|
\quad\text{almost surely for all } t.
\]
\end{restatable}
\vspace{-0.6em}

Lemma~\ref{lemma:onestep_contraction} establishes a sufficient condition under which the mixed prediction error strictly decreases under a single-step Q-Delta update. However, in practice, both $\beta_t$ and $a_t$ are data-dependent and vary across timesteps, so a single analytic bound on $\beta_t a_t$ cannot be determined.
Nevertheless, empirical measures show that $\beta_t a_t$ consistently stays within the contraction regime during training. Figure~\ref{fig:stability_analyses}-(a) shows the distribution of $\beta_t a_t$ collected from the full training steps on 15B tokens across all layers of a 340M Q-Delta model, where values are tightly concentrated within the range of contraction $\beta_t a_t \in (0, 2)$ with mean 0.043.

\begin{figure}[h]
    \centering
    \vspace{-0.5em}
      \includegraphics[width=\linewidth]{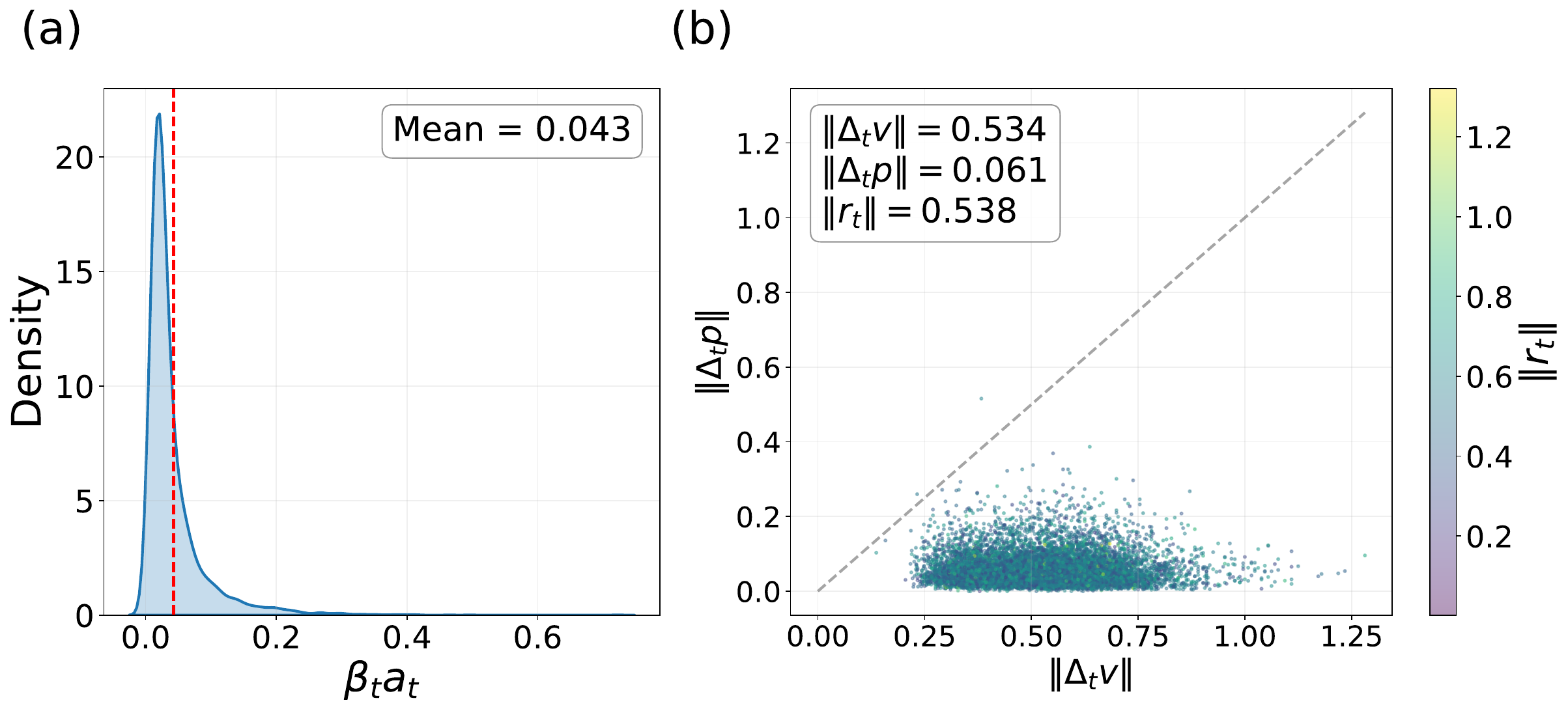}
      \caption{Empirical stability analyses of Q-Delta dynamics. (a): Distribution of $\beta_t a_t$, (b): Scatter plot for $\|\Delta_t v\|$, $\|\Delta_t p\|$, $\|r_t\|$}
    \label{fig:stability_analyses}
    \vspace{-0.5em}
\end{figure}

Building on the one-step contraction result in Lemma~\ref{lemma:onestep_contraction}, we further establish a global stability tracking for Q-Delta, showing that the mixed readout error shows geometric decay and remains uniformly bounded over time, with the bound proportional to the magnitude of residual drifts consisting of target drift and prediction drift.

\begin{restatable}[Global stability and geometric tracking of Q-Delta]{theorem}{globalcontracprop}
\label{thm:qdelta_stability_tracking}
Suppose the single-step contraction condition of Lemma~\ref{lemma:onestep_contraction}
holds with constant $\rho\in(0,1)$, i.e.,
\[
\|v_t - S_t x_t\| \le \rho\,\|v_t - S_{t-1}x_t\|
\qquad \text{a.s. for all } t\ge 1.
\]
Define the pre-update and post-update prediction errors
\[
\tilde e_t := v_t - S_{t-1}x_t,
\qquad
e_t := v_t - S_t x_t.
\]
Define the residual drift
\vspace{-0.5em}
\[
r_t \;:=\; \Delta_t v - \Delta_tp,
\vspace{-0.5em}
\]
where the $\Delta_t v := v_t - v_{t-1}$ is target drift and $\Delta_t p := S_{t-1}(x_t - x_{t-1})$ is prediction drift.
Assume $r_t$ is uniformly bounded, then there exists $r<\infty$ such that
\[
\|r_t\| \le r
\qquad \text{for all } t\ge 1.
\]
Then, almost surely for all $t\ge 1$,
\[
\|e_t\|
\le \rho\|e_{t-1}\|+\rho r 
\;\le\;
\rho^t\|e_0\| + \frac{1-\rho^t}{1-\rho}\,\rho r.
\]
\end{restatable}

Theorem~\ref{thm:qdelta_stability_tracking} establishes that Q-Delta induces a stable global tracking dynamics on mixed key–query prediction errors.
As long as the one-step contraction condition $\beta_t a_t \in (0, 2)$ holds, the mixed readout error decays geometrically up to a bounded radius whose size is controlled by the magnitude of the residual drift $r_t$, which aggregates both target drift $\Delta_t v$ and prediction drift $\Delta_t p$. Intuitively, $\Delta_t p := S_{t-1}\Delta x_t$ captures how changes in the mixed input $(k_t,q_t)$ induce variation in the model’s joint prediction through the accumulated memory state.

\Cref{fig:stability_analyses}-(b) visualizes the relationship between these two drift terms $\Delta_t v$ and $\Delta_t p$, showing that prediction drift is typically smaller in magnitude than the corresponding target drift and remains concentrated near zero. This empirical behavior indicates that residual drift is largely dominated by target variation rather than prediction instability driven by readout key drift.
Figure~\ref{fig:stability_analyses}-(b) also implies that the residual drift terms in steady-state bound given in Theorem~\ref{thm:qdelta_stability_tracking} are well-controlled in magnitude within range (0, 1.2), with its mean norm 0.538, yielding a tight and practically useful  characterization of Q-Delta’s error contraction dynamics.

Taken together, Q-Delta behaves as a stable online learner, ensuring transient error contraction and long-horizon stability under  empirically verified conditions. 
This provides theoretical support for incorporating query-conditioned feedback into state evolution, and justifies its use as a principled state evolution mechanism beyond pure key–value association. Proofs for Lemma~\ref{lemma:onestep_contraction} and Theorem~\ref{thm:qdelta_stability_tracking} are in Appendix~\ref{app:stability_proof}.

\vspace{0.3em}
\section{Experiments}

\subsection{Experimental Setup}

All models are implemented based on pretraining framework {flash-linear-attention}~\citep{yang2024fla}.
We consider two model scales, 340M and 1.3B where the 340M models are pretrained on 15B tokens from the FineWeb-Edu ~\citep{penedo2024finewebdatasetsdecantingweb}, while the 1.3B models are pretrained on 30B tokens.
Training is performed on 4 NVIDIA RTX Pro 6000 (Blackwell) GPUs using mixed-precision arithmetic with bfloat16. 
We compare Q-Delta against RetNet~\citep{sun2023retentivenetworksuccessortransformer}, Mamba~\citep{gu2024mambalineartimesequencemodeling}, Mamba2~\citep{dao2024transformersssmsgeneralizedmodels}, DeltaNet~\citep{yang2025parallelizinglineartransformersdelta}, and GatedDeltaNet~\citep{yang2025gateddeltanetworksimproving}.
All baselines are reproduced on the same framework and trained under matched optimization settings to ensure fair comparison.
We use the AdamW optimizer with cosine learning rate scheduling and gradient clipping, with a peak learning rate of $1\times10^{-3}$ for 340M models and $4\times10^{-4}$ for 1.3B models.

Figure~\ref{fig:train_results}-(a) shows the training loss curves of 340M-parameter models pretrained on 15B tokens.
Q-Delta exhibits stable optimization behavior throughout training and achieves comparable or lower training loss relative to prior linear attention and state-space baselines.
The zoomed region within box highlights the early training phase, where Q-Delta follows comparable or even faster convergence trajectory without introducing optimization instability.

\begin{figure}[t]
    \centering
    \includegraphics[width=0.5\textwidth]{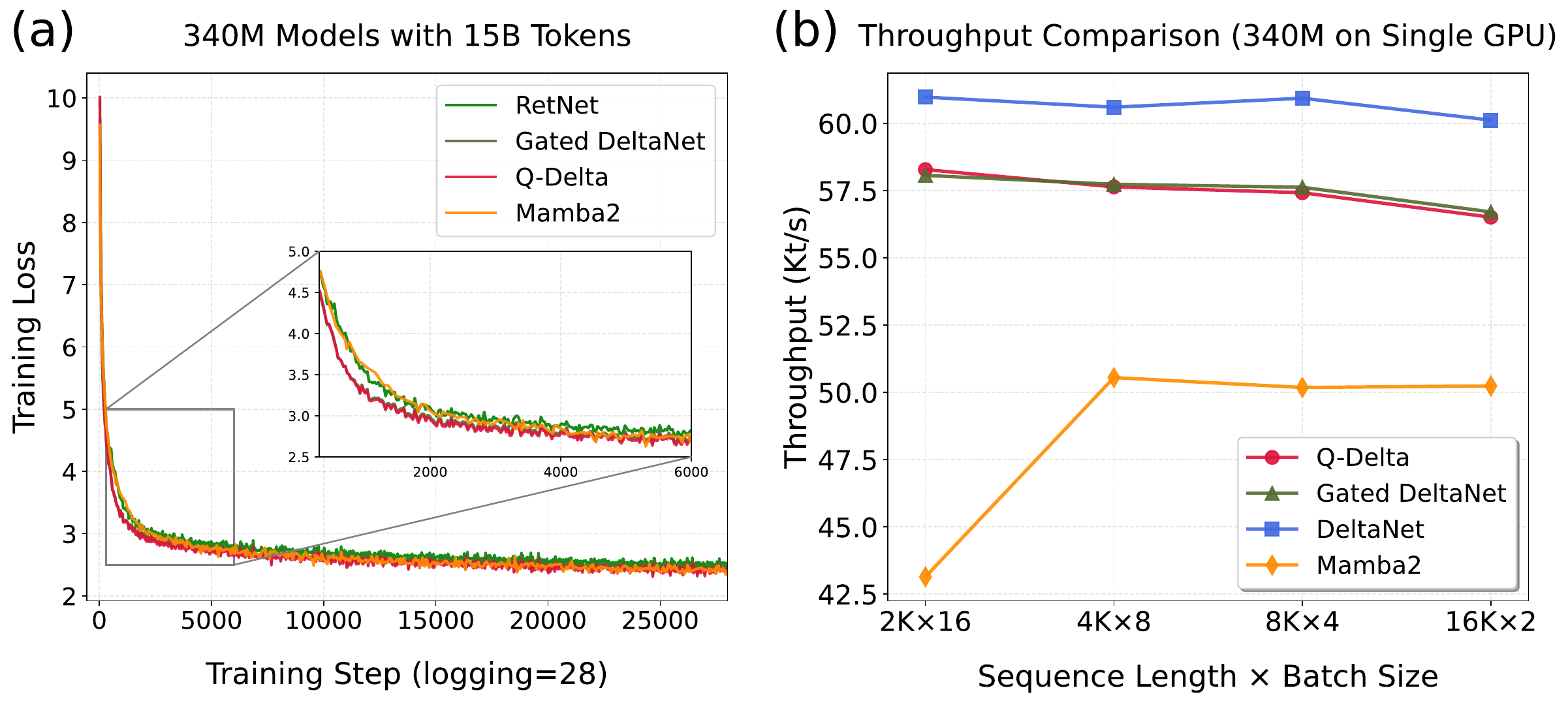}
    \caption{Training results on 340M Models. (a): Train loss curves over 28{,}600 steps (logging interval 28), with the box highlighting early-phase. (b): Single-GPU training throughput comparison across varying sequence length $\times$ batch size configurations.}
    \vspace{-2em}
    \label{fig:train_results}
\end{figure}

Figure~\ref{fig:train_results}-(b) reports a single-GPU throughput comparison on 340M-parameter models.
Throughput is measured as tokens processed per second by running 50 training steps on a single GPU, while varying the sequence length and batch size such that the total token count per step remains constant.
We evaluate configurations from $(2048,16)$ to $(16384,2)$, matching increased sequence lengths with proportionally reduced batch sizes.
Q-Delta achieves consistently high throughput across all configurations, closely matching delta-based baselines.
In contrast, Mamba2 exhibits noticeably lower throughput, particularly at shorter sequence lengths.
Overall, the results indicate that Q-Delta preserves the computational efficiency of delta-rule architectures while scaling robustly to longer sequence under practical training settings.

\begin{table*}[t]
\centering
\scriptsize
\setlength{\tabcolsep}{4pt}
\setlength{\arrayrulewidth}{0.2pt}
\caption{Zero-shot performance comparison of 340M and 1.3B models trained on FineWeb-Edu~\citep{penedo2024finewebdatasetsdecantingweb}. The commonsense Reasoning task is evaluated by lm-evaluation-harness~\citep{eval-harness}. All reproduced by us. Best in \textbf{bold} and second-best \secondbest{underlined}.}
\vspace{0.3em}
\label{tbl:main_results}
\resizebox{1.0\textwidth}{!}{%
\begin{tabular}{lccccccccccc}
\toprule
Model
& Lamb ppl. $\downarrow$
& Wiki ppl. $\downarrow$
& ARC$_\text{E}$ $\uparrow$
& ARC$_\text{C}$ $\uparrow$
& Hella. $\uparrow$
& Lamb. $\uparrow$
& PIQA $\uparrow$
& Wino. $\uparrow$
& BoolQ $\uparrow$
& OpenBook $\uparrow$
& Avg. $\uparrow$ \\
\midrule
\multicolumn{12}{l}{\textit{340M parameters, 15B training tokens}} \\
\midrule
RetNet
& 52.29
& 31.36
& 57.07
& 28.41
& 38.71
& 27.36
& 66.54
& 49.41
& 56.88
& 32.00
& 44.55 \\

Mamba
& \secondbest{31.20}
& 27.50
& \secondbest{59.68}
& \textbf{29.18}
& \textbf{42.98}
& 32.97
& {67.52}
& 51.78
& 55.81
& {33.20}
& \secondbest{46.64} \\

Mamba2
& \textbf{30.35}
& \textbf{26.60}
& 59.30
& \secondbest{29.01}
& \secondbest{42.13}
& \secondbest{33.67}
& \textbf{68.01}
& \secondbest{52.33}
& 51.90
& \secondbest{33.80}
& 46.27 \\

DeltaNet
& 63.04
& 28.78
& 54.88
& 27.47
& 38.65
& 27.11
& 63.60
& 49.96
& \secondbest{59.46}
& 29.40
& 43.82 \\

Gated DeltaNet
& 36.27
& 27.82
& \textbf{60.02}
& 25.94
& 40.25
& 31.52
& 67.30
& 51.54
& 57.13
& \textbf{34.40}
& 46.01 \\

\cellcolor{qdelta}\textbf{Q-Delta}
& \cellcolor{qdelta}{32.67}
& \cellcolor{qdelta}\secondbest{26.89}
& \cellcolor{qdelta}{59.51}
& \cellcolor{qdelta}{28.50}
& \cellcolor{qdelta}{41.61}
& \cellcolor{qdelta}\textbf{33.90}
& \cellcolor{qdelta}\secondbest{67.63}
& \cellcolor{qdelta}\textbf{52.88}
& \cellcolor{qdelta}\textbf{59.48}
& \cellcolor{qdelta}\textbf{34.40}
& \cellcolor{qdelta}\textbf{47.24} \\

\midrule
\multicolumn{12}{l}{\textit{1.3B parameters, 30B training tokens}} \\
\midrule
RetNet
& 21.84
& 22.45
& 63.68
& 33.36
& 47.73
& 38.70
& 69.04
& 52.72
& 60.61
& 36.60
&  50.31 \\

Mamba
& 16.98 
& 19.89
& 68.10
& 36.18
& \secondbest{53.44}
& 40.77
& \textbf{72.20}
& \secondbest{55.01}
& 55.63
& 37.80
& 52.39 \\

Mamba2
& 17.40
& \secondbest{19.47}
& \textbf{69.87}
& \secondbest{36.35}
& 53.24
& {40.68}
& 70.29
& \textbf{56.04}
& 55.81
& 37.40
& 52.46 \\

DeltaNet
& 16.64
& 19.77
& 67.63
& 34.47
& 51.09
& 41.78
& {70.95}
& 54.70
& \secondbest{61.19}
& \secondbest{38.40}
& {52.53} \\

Gated DeltaNet
& \secondbest{15.32}
& 19.61
& \secondbest{68.60}
& 33.28
& {52.60}
& \textbf{43.80}
& {70.84}
& 54.78
& 59.42
& \textbf{38.80}
& \secondbest{52.77} \\

\cellcolor{qdelta}\textbf{Q-Delta}
& \cellcolor{qdelta}\textbf{15.19}
& \cellcolor{qdelta}\textbf{19.21}
& \cellcolor{qdelta}{68.27}
& \cellcolor{qdelta}\textbf{36.60}
& \cellcolor{qdelta}\textbf{53.46}
& \cellcolor{qdelta}\secondbest{43.28}
& \cellcolor{qdelta}\secondbest{71.44}
& \cellcolor{qdelta}{54.93}
& \cellcolor{qdelta}\textbf{61.41}
& \cellcolor{qdelta}\secondbest{38.40}
& \cellcolor{qdelta}\textbf{53.47} \\

\bottomrule
\end{tabular}
}
\end{table*}

\vspace{-0.3em}
\subsection{Evaluation}

\vspace{-0.2em}
\paragraph{Language Modeling.}
We evaluate commonsense reasoning performance using LM Evaluation Harness~\citep{lintang_sutawika_2024_12608602} to test zero-shot language modeling capacity.
Following standard practice, we report language modeling perplexity on LAMBADA~\citep{paperno2016lambadadatasetwordprediction} and Wikitext~\citep{merity2016pointersentinelmixturemodels}, and zero-shot accuracy on multiple-choice reasoning benchmarks, including BoolQ~\citep{clark-etal-2019-boolq}, HellaSwag~\citep{zellers2019hellaswagmachinereallyfinish}, PIQA~\citep{bisk2019piqareasoningphysicalcommonsense}, Arc-Easy, Arc-Challenge~\citep{clark2018thinksolvedquestionanswering},  WinoGrande~\citep{sakaguchi2019winograndeadversarialwinogradschema}, and OpenBookQA~\citep{mihaylov2018suitarmorconductelectricity}.

From Table~\ref{tbl:main_results}, across both model scales, Q-Delta consistently achieves strong zero-shot performance on commonsense reasoning benchmarks.
At the 340M scale, Q-Delta attains the best average accuracy and improves over other baselines on various reasoning tasks, while maintaining competitive perplexity on both WikiText and LAMBADA.
At 1.3B scale, Q-Delta further strengthens this trend, achieving the highest average score and leading performance on several benchmarks, notably on language modeling tasks, ARC-Challenge, HellaSwag, and BoolQ.
These results indicate that incorporating query-conditioned feedback into state evolution improves both language modeling and zero-shot reasoning ability without task-specific adaptation.

\begin{table*}[t]
\centering
\small
\setlength{\tabcolsep}{5pt}
\caption{Retrieval performance on the synthetic S-NIAH benchmark from RULER~\citet{hsieh2024rulerwhatsrealcontext}, evaluated on 1.3B models.
Results are reported under varying context lengths (1K, 2K, and 4K tokens).
Best results are shown in \textbf{bold} and second-best in \secondbest{underlined}.}
\label{tbl:sniah_1p3b}
\resizebox{0.90\textwidth}{!}{%
\begin{tabular}{lCCCCCCCCCC}
\toprule
\multirow{2}{*}{{Model}}
& \multicolumn{3}{c}{\textbf{S-NIAH-1} (pass-key retrieval)}
& \multicolumn{3}{c}{\textbf{S-NIAH-2} (number in haystack)}
& \multicolumn{3}{c}{\textbf{S-NIAH-3} (uuid in haystack)}
& \multirow{2}{*}{{Avg.}} \\
\cmidrule(lr){2-4} \cmidrule(lr){5-7} \cmidrule(lr){8-10}
 & 1K & 2K & 4K
 & 1K & 2K & 4K
 & 1K & 2K & 4K \\
\midrule

RetNet
 & 96.6 & 27.8 & 7.4
 & 99.4 & 60.8 & 24.4
 & 20.0 & 5.2 & 1.2
 & 38.09 \\

Mamba
 & 99.8 & 99.6 & 87.0
 & 98.8 & 92.8 & 50.8
 & 22.0 & 12.0 & 0.8
 & {62.62} \\

Mamba2
 & \textbf{100.0} & \secondbest{99.8} & \secondbest{99.0}
 & \secondbest{99.8} & {95.4} & {57.0}
 & 76.0 & 50.6 & 11.6
 & 76.58 \\

DeltaNet
 & \textbf{100.0} & \textbf{100.0} & \textbf{100.0}
 & \secondbest{99.8} & 93.6 & 49.6
 & \secondbest{87.4} & \textbf{75.8} & \secondbest{25.4}
 & {81.29} \\

Gated DeltaNet
 & \textbf{100.0} & \textbf{100.0} & \textbf{100.0}
 & {100.0} & \textbf{99.8} & \secondbest{76.6}
 & {83.8} & {70.0} & {21.4}
 & \secondbest{83.51} \\
\cellcolor{qdelta}\textbf{{Q-Delta}}
 & \cellcolor{qdelta}\textbf{100.0} & \cellcolor{qdelta}\textbf{100.0} & \cellcolor{qdelta}{\textbf{100.0}}
 & \cellcolor{qdelta}\textbf{100.0} & \cellcolor{qdelta}\secondbest{99.4} & \cellcolor{qdelta}\textbf{94.2}
 & \cellcolor{qdelta}\textbf{94.6} & \cellcolor{qdelta}\secondbest{74.0} & \cellcolor{qdelta}\textbf{48.0}
 & \cellcolor{qdelta}\textbf{90.02} \\
\bottomrule
\end{tabular}
}
\end{table*}

\paragraph{Real and synthetic retrieval.}
We evaluate retrieval capabilities using both real-world and synthetic benchmarks.
For real-world retrieval, we adopt the recall-intensive tasks~\citep{arora2024just} and evaluate on 340M models.
All real-world retrieval inputs are truncated to a maximum context length of 2K tokens.
For synthetic retrieval, we evaluate 1.3B-parameter models on the S-NIAH (Synthetic Needle-In-A-Haystack) benchmark~\citep{hsieh2024rulerwhatsrealcontext}, which measures model’s ability to retrieve sparse target information embedded at varying positions within long contexts.
We report results under context lengths of 1K, 2K, and 4K tokens to assess generalization beyond the training context.
Together, these benchmarks evaluate complementary aspects of retrieval, ranging from structured real-world recall to controlled long-context information extraction.

\begin{table}[t]
\centering
\small
\setlength{\tabcolsep}{7pt}
\caption{Retrieval performance on real-world recall-intensive tasks from~\citet{arora2024just}, evaluated with 340M-parameter models.
All inputs are truncated to a context length of 2K tokens and formatted in a cloze-style next-token prediction setting.
Best results are shown in \textbf{bold} and second-best in \secondbest{underlined}.}
\vspace{0.2em}
\label{tbl:pla_twice_340m}
\resizebox{0.49\textwidth}{!}{%
\begin{tabular}{lcccccc}
\toprule
Models & SWDE & SQD & FDA & TQA & NQ & Drop \\
\midrule
Mamba
 & 17.1 & 43.6 & 6.4
 & \textbf{55.2} & {17.5}  & 26.4 \\
Mamba2   
& \underline{29.1} 
& \underline{55.3} 
& 18.4 
& {49.1} 
& \secondbest{18.2}
& \underline{33.4} \\

DeltaNet 
& 22.9 
& 51.5 
& 16.7 
& 45.0 
& 15.2 
& 28.2 \\

Gated DeltaNet  
& \underline{29.1} 
& \textbf{56.0} 
& \underline{18.9} 
& \underline{48.9} 
& \secondbest{18.2} 
& \textbf{34.2} \\

\cellcolor{qdelta}\textbf{Q-Delta} 
& \cellcolor{qdelta}\textbf{31.6} 
& \cellcolor{qdelta}52.3 
& \cellcolor{qdelta}\textbf{22.0} 
& \cellcolor{qdelta}48.6 
& \cellcolor{qdelta}\textbf{18.4} 
& \cellcolor{qdelta}33.2 \\

\bottomrule
\end{tabular}
}
\end{table}

On real-world recall-intensive tasks (Table~\ref{tbl:pla_twice_340m}), Q-Delta consistently matches or outperforms prior linear RNN models, achieving the best average score across tasks.
These results suggest that incorporating query-conditioned feedback into state evolution improves both controlled synthetic retrieval and practical real-world recall, while maintaining linear-time scalability.
On the synthetic S-NIAH benchmark (Table~\ref{tbl:sniah_1p3b}), Q-Delta achieves the highest average accuracy among all linear recurrent baselines, with near-perfect performance on pass-key retrieval (S-NIAH-1) across all evaluated context lengths.
Notably, Q-Delta substantially improves performance on the more challenging number-in-haystack and UUID-in-haystack tasks (S-NIAH-2 and S-NIAH-3), particularly at longer contexts up to 4K tokens, indicating stronger robustness to sparse information retrieval especially as context length increases.

\begin{table}[h]
\centering
\caption{Ablation on the query-feedback coefficient $\lambda$ for 340M Q-Delta. {Scalar $\lambda$} tests fixed query-feedback strength versus adaptive learnable modulation, {No state decay} removes the recurrent forget/decay gate ($\alpha_t=1$), and {No gating} uses full query correction ($\lambda_t=1$), disabling adaptive gating of the query-feedback term.}
\vspace{0.5em}
\label{tbl:ablation_lmb}
\resizebox{0.48\textwidth}{!}{%
\begin{tabular}{lccc}
\toprule
$\lambda$ & \textbf{Wiki ppl. $\downarrow$}  & \textbf{Lamb ppl. $\downarrow$} & \textbf{Avg Acc. (8 tasks) $\uparrow$} \\
\midrule
\cellcolor{qdelta}Learnable $\lambda_t$ (Q-Delta) 
& \cellcolor{qdelta}26.89 & \cellcolor{qdelta}{32.67}  & \cellcolor{qdelta}{47.24} \\
{Scalar $\lambda = 0.2$} & {26.96} & {35.39} & {46.99} \\
Scalar $\lambda = 0.5$ & 26.86 & 33.31 & 47.20 \\
Scalar $\lambda = 0.8$ & 26.61 & 33.58 & 46.42 \\
{No state decay ($\alpha_t = 1.0$\,)} & {26.52} & {32.97} & {45.86} \\
No gating ($\lambda_t = 1.0$\,)  & 26.55 & 35.21 & 46.36 \\
\bottomrule
\vspace{-2em}
\end{tabular}}
\end{table}

\paragraph{Ablation Studies.}
Given Q-Delta recurrence rule
$S_t = \alpha_t S_{t-1}\big(I - \beta_t(k_t + \lambda_t q_t)k_t^\top\big) + \beta_t v_t k_t^\top$,
the query-conditioned correction is governed solely by $\lambda_t$, which scales the
query-feedback term $\lambda_t q_t k_t^\top$. Since the query-based state correction is
our central contribution, $\lambda_t$ is the natural ablation target, and we additionally
ablate the decay factor $\alpha_t$ (Table~\ref{tbl:ablation_lmb}).
Across fixed scalar values, performance is relatively robust, with $\lambda=0.5$ giving
the best scalar result (47.20 average accuracy) while learning $\lambda_t$ end-to-end yields
the best overall performance, maintaining strong perplexity. This indicates that
adaptively modulating the query-feedback strength is beneficial.

Removing the decay factor ($\alpha_t=1$) lowers accuracy to $45.86$ but remains clearly
above its most direct non-decay baseline, DeltaNet ($43.82$ at the same 340M scale),
indicating that query feedback is beneficial independent of the gating mechanism.
Overall, query feedback contributes consistently across settings, and allowing the
model to tune its strength gives the best trade-off.

\section{Conclusion}
This work reconsiders a core assumption in linear attention and recurrent sequence
models: that state evolution is governed solely by key--value association, with queries
confined to passive readout. We observe that the query is the direction along which the state is
read out, so the query-conditioned value prediction $\hat{o}_t = S_{t-1}q_t$ is a readout-aligned
state correction signal that conventional delta-rule updates leave uncorrected. Building on this, we propose
Q-Delta, a query-aware delta rule that injects this prediction error into state evolution
while preserving linear-time efficiency, and we establish theoretical justification that the resulting mixed
key--query dynamics are stable under mild, empirically verified conditions. Q-Delta
consistently improves over strong linear-attention and SSM baselines, showing that
incorporating query into recurrent state update is an effective way to move beyond pure key--value association.

\clearpage

\section*{Acknowledgements}
This work was partly supported by the Institute for Information \& Communications Technology Planning \& Evaluation (IITP) grants funded by the Korean government (MSIT) (No. RS-2026-25526850, High-Efficiency Neural Networks for Artificial General Intelligence, 33\%; No.2022-0-00857, Development of Financial and Economic Digital Twin Platform based on AI and Data, 33\%; No. RS-2025-25442149, LG AI STAR Talent Development Program for Leading Large-Scale Generative AI Models in the Physical AI Domain, 1\%), and Samsung Research Funding \& Incubation Center of Samsung Electronics under Project Number SRFC-IT2402-08, 33\%.

\section*{Impact Statement}
This work proposes Q-Delta, a novel query-aware delta rule that enriches linear-time sequential models, enabling rich state dynamics by integrating complementary key–query signals into state evolution.
The research supports scalable language modeling and long-context applications, improving expressivity and interpretability of linear attention frameworks. As with other large-scale sequence models, potential risks such as misuse for misinformation generation or unintended memorization of sensitive data may arise in wide applications, so maintaining responsible training, evaluation, and deployment practices is important.

\bibliographystyle{icml2026}
\bibliography{references}

\begin{thebibliography}{45}
\providecommand{\natexlab}[1]{#1}
\providecommand{\url}[1]{\texttt{#1}}
\expandafter\ifx\csname urlstyle\endcsname\relax
  \providecommand{\doi}[1]{doi: #1}\else
  \providecommand{\doi}{doi: \begingroup \urlstyle{rm}\Url}\fi

\bibitem[Arora et~al.(2023)Arora, Eyuboglu, Timalsina, Johnson, Poli, Zou, Rudra, and Ré]{arora2023zoologymeasuringimprovingrecall}
Arora, S., Eyuboglu, S., Timalsina, A., Johnson, I., Poli, M., Zou, J., Rudra, A., and Ré, C.
\newblock Zoology: Measuring and improving recall in efficient language models, 2023.
\newblock URL \url{https://arxiv.org/abs/2312.04927}.

\bibitem[Arora et~al.(2024)Arora, Timalsina, Singhal, Eyuboglu, Zhao, Rao, Rudra, and Ré]{arora2024just}
Arora, S., Timalsina, A., Singhal, A., Eyuboglu, S., Zhao, X., Rao, A., Rudra, A., and Ré, C.
\newblock Just read twice: closing the recall gap for recurrent language models.
\newblock 2024.

\bibitem[Behrouz et~al.(2024)Behrouz, Zhong, and Mirrokni]{behrouz2024titanslearningmemorizetest}
Behrouz, A., Zhong, P., and Mirrokni, V.
\newblock Titans: Learning to memorize at test time, 2024.
\newblock URL \url{https://arxiv.org/abs/2501.00663}.

\bibitem[Bischof \& Loan(1985)Bischof and Loan]{Bischof1985TheWR}
Bischof, C.~H. and Loan, C.~V.
\newblock The wy representation for products of householder matrices.
\newblock In \emph{PP}, 1985.
\newblock URL \url{https://api.semanticscholar.org/CorpusID:36094006}.

\bibitem[Bisk et~al.(2019)Bisk, Zellers, Bras, Gao, and Choi]{bisk2019piqareasoningphysicalcommonsense}
Bisk, Y., Zellers, R., Bras, R.~L., Gao, J., and Choi, Y.
\newblock Piqa: Reasoning about physical commonsense in natural language, 2019.
\newblock URL \url{https://arxiv.org/abs/1911.11641}.

\bibitem[Brown et~al.(2020)Brown, Mann, Ryder, Subbiah, Kaplan, Dhariwal, Neelakantan, Shyam, Sastry, Askell, Agarwal, Herbert-Voss, Krueger, Henighan, Child, Ramesh, Ziegler, Wu, Winter, Hesse, Chen, Sigler, Litwin, Gray, Chess, Clark, Berner, McCandlish, Radford, Sutskever, and Amodei]{brown2020languagemodelsfewshotlearners}
Brown, T.~B., Mann, B., Ryder, N., Subbiah, M., Kaplan, J., Dhariwal, P., Neelakantan, A., Shyam, P., Sastry, G., Askell, A., Agarwal, S., Herbert-Voss, A., Krueger, G., Henighan, T., Child, R., Ramesh, A., Ziegler, D.~M., Wu, J., Winter, C., Hesse, C., Chen, M., Sigler, E., Litwin, M., Gray, S., Chess, B., Clark, J., Berner, C., McCandlish, S., Radford, A., Sutskever, I., and Amodei, D.
\newblock Language models are few-shot learners, 2020.
\newblock URL \url{https://arxiv.org/abs/2005.14165}.

\bibitem[Chevalier(2018)]{chevalier2018larnnlinearattentionrecurrent}
Chevalier, G.
\newblock Larnn: Linear attention recurrent neural network, 2018.
\newblock URL \url{https://arxiv.org/abs/1808.05578}.

\bibitem[Clark et~al.(2019)Clark, Lee, Chang, Kwiatkowski, Collins, and Toutanova]{clark-etal-2019-boolq}
Clark, C., Lee, K., Chang, M.-W., Kwiatkowski, T., Collins, M., and Toutanova, K.
\newblock {B}ool{Q}: Exploring the surprising difficulty of natural yes/no questions.
\newblock In Burstein, J., Doran, C., and Solorio, T. (eds.), \emph{Proceedings of the 2019 Conference of the North {A}merican Chapter of the Association for Computational Linguistics: Human Language Technologies, Volume 1 (Long and Short Papers)}, pp.\  2924--2936, Minneapolis, Minnesota, June 2019. Association for Computational Linguistics.
\newblock \doi{10.18653/v1/N19-1300}.
\newblock URL \url{https://aclanthology.org/N19-1300/}.

\bibitem[Clark et~al.(2018)Clark, Cowhey, Etzioni, Khot, Sabharwal, Schoenick, and Tafjord]{clark2018thinksolvedquestionanswering}
Clark, P., Cowhey, I., Etzioni, O., Khot, T., Sabharwal, A., Schoenick, C., and Tafjord, O.
\newblock Think you have solved question answering? try arc, the ai2 reasoning challenge, 2018.
\newblock URL \url{https://arxiv.org/abs/1803.05457}.

\bibitem[Dao \& Gu(2024)Dao and Gu]{dao2024transformersssmsgeneralizedmodels}
Dao, T. and Gu, A.
\newblock Transformers are ssms: Generalized models and efficient algorithms through structured state space duality, 2024.
\newblock URL \url{https://arxiv.org/abs/2405.21060}.

\bibitem[Dominguez \& Orti(2018)Dominguez and Orti]{dominguez2018fast}
Dominguez, A. E.~T. and Orti, E. S.~Q.
\newblock Fast blocking of householder reflectors on graphics processors.
\newblock In \emph{2018 26th Euromicro International Conference on Parallel, Distributed and Network-based Processing (PDP)}, pp.\  385--393. IEEE, 2018.

\bibitem[Dua et~al.(2019)Dua, Wang, Dasigi, Stanovsky, Singh, and Gardner]{dua-etal-2019-drop}
Dua, D., Wang, Y., Dasigi, P., Stanovsky, G., Singh, S., and Gardner, M.
\newblock {DROP}: A reading comprehension benchmark requiring discrete reasoning over paragraphs.
\newblock In Burstein, J., Doran, C., and Solorio, T. (eds.), \emph{Proceedings of the 2019 Conference of the North {A}merican Chapter of the Association for Computational Linguistics: Human Language Technologies, Volume 1 (Long and Short Papers)}, pp.\  2368--2378, Minneapolis, Minnesota, June 2019. Association for Computational Linguistics.
\newblock \doi{10.18653/v1/N19-1246}.
\newblock URL \url{https://aclanthology.org/N19-1246/}.

\bibitem[Gao et~al.(2024)Gao, Tow, Abbasi, Biderman, Black, DiPofi, Foster, Golding, Hsu, Le~Noac'h, Li, McDonell, Muennighoff, Ociepa, Phang, Reynolds, Schoelkopf, Skowron, Sutawika, Tang, Thite, Wang, Wang, and Zou]{eval-harness}
Gao, L., Tow, J., Abbasi, B., Biderman, S., Black, S., DiPofi, A., Foster, C., Golding, L., Hsu, J., Le~Noac'h, A., Li, H., McDonell, K., Muennighoff, N., Ociepa, C., Phang, J., Reynolds, L., Schoelkopf, H., Skowron, A., Sutawika, L., Tang, E., Thite, A., Wang, B., Wang, K., and Zou, A.
\newblock The language model evaluation harness, 07 2024.
\newblock URL \url{https://zenodo.org/records/12608602}.

\bibitem[Gu \& Dao(2024)Gu and Dao]{gu2024mambalineartimesequencemodeling}
Gu, A. and Dao, T.
\newblock Mamba: Linear-time sequence modeling with selective state spaces, 2024.
\newblock URL \url{https://arxiv.org/abs/2312.00752}.

\bibitem[Gu et~al.(2020)Gu, Dao, Ermon, Rudra, and Re]{gu2020hipporecurrentmemoryoptimal}
Gu, A., Dao, T., Ermon, S., Rudra, A., and Re, C.
\newblock Hippo: Recurrent memory with optimal polynomial projections, 2020.
\newblock URL \url{https://arxiv.org/abs/2008.07669}.

\bibitem[Hsieh et~al.(2024)Hsieh, Sun, Kriman, Acharya, Rekesh, Jia, Zhang, and Ginsburg]{hsieh2024rulerwhatsrealcontext}
Hsieh, C.-P., Sun, S., Kriman, S., Acharya, S., Rekesh, D., Jia, F., Zhang, Y., and Ginsburg, B.
\newblock Ruler: What's the real context size of your long-context language models?, 2024.
\newblock URL \url{https://arxiv.org/abs/2404.06654}.

\bibitem[Hu et~al.(2025)Hu, Pan, Du, Lan, Tang, Wen, Liang, and Sun]{hu2025combaimprovingbilinearrnns}
Hu, J., Pan, Y., Du, J., Lan, D., Tang, X., Wen, Q., Liang, Y., and Sun, W.
\newblock Comba: Improving bilinear rnns with closed-loop control, 2025.
\newblock URL \url{https://arxiv.org/abs/2506.02475}.

\bibitem[Joffrain et~al.(2006)Joffrain, Low, Quintana-Ort{\'\i}, Geijn, and Zee]{joffrain2006accumulating}
Joffrain, T., Low, T.~M., Quintana-Ort{\'\i}, E.~S., Geijn, R. v.~d., and Zee, F. G.~V.
\newblock Accumulating householder transformations, revisited.
\newblock \emph{ACM Transactions on Mathematical Software (TOMS)}, 32\penalty0 (2):\penalty0 169--179, 2006.

\bibitem[Joshi et~al.(2017)Joshi, Choi, Weld, and Zettlemoyer]{joshi-etal-2017-triviaqa}
Joshi, M., Choi, E., Weld, D., and Zettlemoyer, L.
\newblock {T}rivia{QA}: A large scale distantly supervised challenge dataset for reading comprehension.
\newblock In Barzilay, R. and Kan, M.-Y. (eds.), \emph{Proceedings of the 55th Annual Meeting of the Association for Computational Linguistics (Volume 1: Long Papers)}, pp.\  1601--1611, Vancouver, Canada, July 2017. Association for Computational Linguistics.
\newblock \doi{10.18653/v1/P17-1147}.
\newblock URL \url{https://aclanthology.org/P17-1147/}.

\bibitem[Katharopoulos et~al.(2020)Katharopoulos, Vyas, Pappas, and Fleuret]{katharopoulos2020transformersrnnsfastautoregressive}
Katharopoulos, A., Vyas, A., Pappas, N., and Fleuret, F.
\newblock Transformers are rnns: Fast autoregressive transformers with linear attention, 2020.
\newblock URL \url{https://arxiv.org/abs/2006.16236}.

\bibitem[Kitaev et~al.(2020)Kitaev, Łukasz Kaiser, and Levskaya]{kitaev2020reformerefficienttransformer}
Kitaev, N., Łukasz Kaiser, and Levskaya, A.
\newblock Reformer: The efficient transformer, 2020.
\newblock URL \url{https://arxiv.org/abs/2001.04451}.

\bibitem[Kwiatkowski et~al.(2019)Kwiatkowski, Palomaki, Redfield, Collins, Parikh, Alberti, Epstein, Polosukhin, Devlin, Lee, Toutanova, Jones, Kelcey, Chang, Dai, Uszkoreit, Le, and Petrov]{kwiatkowski-etal-2019-natural}
Kwiatkowski, T., Palomaki, J., Redfield, O., Collins, M., Parikh, A., Alberti, C., Epstein, D., Polosukhin, I., Devlin, J., Lee, K., Toutanova, K., Jones, L., Kelcey, M., Chang, M.-W., Dai, A.~M., Uszkoreit, J., Le, Q., and Petrov, S.
\newblock Natural questions: A benchmark for question answering research.
\newblock \emph{Transactions of the Association for Computational Linguistics}, 7:\penalty0 452--466, 2019.
\newblock \doi{10.1162/tacl_a_00276}.
\newblock URL \url{https://aclanthology.org/Q19-1026/}.

\bibitem[Liu et~al.(2024)Liu, Wang, Wu, Feng, Stone, and Liu]{liu2024longhornstatespacemodels}
Liu, B., Wang, R., Wu, L., Feng, Y., Stone, P., and Liu, Q.
\newblock Longhorn: State space models are amortized online learners, 2024.
\newblock URL \url{https://arxiv.org/abs/2407.14207}.

\bibitem[Lockard et~al.(2019)Lockard, Shiralkar, and Dong]{lockard-etal-2019-openceres}
Lockard, C., Shiralkar, P., and Dong, X.~L.
\newblock {O}pen{C}eres: {W}hen open information extraction meets the semi-structured web.
\newblock In Burstein, J., Doran, C., and Solorio, T. (eds.), \emph{Proceedings of the 2019 Conference of the North {A}merican Chapter of the Association for Computational Linguistics: Human Language Technologies, Volume 1 (Long and Short Papers)}, pp.\  3047--3056, Minneapolis, Minnesota, June 2019. Association for Computational Linguistics.
\newblock \doi{10.18653/v1/N19-1309}.
\newblock URL \url{https://aclanthology.org/N19-1309/}.

\bibitem[Merity et~al.(2016)Merity, Xiong, Bradbury, and Socher]{merity2016pointersentinelmixturemodels}
Merity, S., Xiong, C., Bradbury, J., and Socher, R.
\newblock Pointer sentinel mixture models, 2016.
\newblock URL \url{https://arxiv.org/abs/1609.07843}.

\bibitem[Mihaylov et~al.(2018)Mihaylov, Clark, Khot, and Sabharwal]{mihaylov2018suitarmorconductelectricity}
Mihaylov, T., Clark, P., Khot, T., and Sabharwal, A.
\newblock Can a suit of armor conduct electricity? a new dataset for open book question answering, 2018.
\newblock URL \url{https://arxiv.org/abs/1809.02789}.

\bibitem[Olsson et~al.(2022)Olsson, Elhage, Nanda, Joseph, DasSarma, Henighan, Mann, Askell, Bai, Chen, Conerly, Drain, Ganguli, Hatfield-Dodds, Hernandez, Johnston, Jones, Kernion, Lovitt, Ndousse, Amodei, Brown, Clark, Kaplan, McCandlish, and Olah]{olsson2022incontextlearninginductionheads}
Olsson, C., Elhage, N., Nanda, N., Joseph, N., DasSarma, N., Henighan, T., Mann, B., Askell, A., Bai, Y., Chen, A., Conerly, T., Drain, D., Ganguli, D., Hatfield-Dodds, Z., Hernandez, D., Johnston, S., Jones, A., Kernion, J., Lovitt, L., Ndousse, K., Amodei, D., Brown, T., Clark, J., Kaplan, J., McCandlish, S., and Olah, C.
\newblock In-context learning and induction heads, 2022.
\newblock URL \url{https://arxiv.org/abs/2209.11895}.

\bibitem[Paperno et~al.(2016)Paperno, Kruszewski, Lazaridou, Pham, Bernardi, Pezzelle, Baroni, Boleda, and Fernández]{paperno2016lambadadatasetwordprediction}
Paperno, D., Kruszewski, G., Lazaridou, A., Pham, Q.~N., Bernardi, R., Pezzelle, S., Baroni, M., Boleda, G., and Fernández, R.
\newblock The lambada dataset: Word prediction requiring a broad discourse context, 2016.
\newblock URL \url{https://arxiv.org/abs/1606.06031}.

\bibitem[Penedo et~al.(2024)Penedo, Kydlíček, allal, Lozhkov, Mitchell, Raffel, Werra, and Wolf]{penedo2024finewebdatasetsdecantingweb}
Penedo, G., Kydlíček, H., allal, L.~B., Lozhkov, A., Mitchell, M., Raffel, C., Werra, L.~V., and Wolf, T.
\newblock The fineweb datasets: Decanting the web for the finest text data at scale, 2024.
\newblock URL \url{https://arxiv.org/abs/2406.17557}.

\bibitem[Rajpurkar et~al.(2018)Rajpurkar, Jia, and Liang]{rajpurkar-etal-2018-know}
Rajpurkar, P., Jia, R., and Liang, P.
\newblock Know what you don{'}t know: Unanswerable questions for {SQ}u{AD}.
\newblock In Gurevych, I. and Miyao, Y. (eds.), \emph{Proceedings of the 56th Annual Meeting of the Association for Computational Linguistics (Volume 2: Short Papers)}, pp.\  784--789, Melbourne, Australia, July 2018. Association for Computational Linguistics.
\newblock \doi{10.18653/v1/P18-2124}.
\newblock URL \url{https://aclanthology.org/P18-2124/}.

\bibitem[Sakaguchi et~al.(2019)Sakaguchi, Bras, Bhagavatula, and Choi]{sakaguchi2019winograndeadversarialwinogradschema}
Sakaguchi, K., Bras, R.~L., Bhagavatula, C., and Choi, Y.
\newblock Winogrande: An adversarial winograd schema challenge at scale, 2019.
\newblock URL \url{https://arxiv.org/abs/1907.10641}.

\bibitem[Schlag et~al.(2021)Schlag, Irie, and Schmidhuber]{pmlr-v139-schlag21a}
Schlag, I., Irie, K., and Schmidhuber, J.
\newblock Linear transformers are secretly fast weight programmers.
\newblock In Meila, M. and Zhang, T. (eds.), \emph{Proceedings of the 38th International Conference on Machine Learning}, volume 139 of \emph{Proceedings of Machine Learning Research}, pp.\  9355--9366. PMLR, 18--24 Jul 2021.
\newblock URL \url{https://proceedings.mlr.press/v139/schlag21a.html}.

\bibitem[Smith et~al.(2023)Smith, Warrington, and Linderman]{smith2023simplifiedstatespacelayers}
Smith, J. T.~H., Warrington, A., and Linderman, S.~W.
\newblock Simplified state space layers for sequence modeling, 2023.
\newblock URL \url{https://arxiv.org/abs/2208.04933}.

\bibitem[Sun et~al.(2023)Sun, Dong, Huang, Ma, Xia, Xue, Wang, and Wei]{sun2023retentivenetworksuccessortransformer}
Sun, Y., Dong, L., Huang, S., Ma, S., Xia, Y., Xue, J., Wang, J., and Wei, F.
\newblock Retentive network: A successor to transformer for large language models, 2023.
\newblock URL \url{https://arxiv.org/abs/2307.08621}.

\bibitem[Sun et~al.(2024)Sun, Dong, Zhu, Huang, Wang, Ma, Zhang, Wang, and Wei]{sun2024cacheoncedecoderdecoderarchitectures}
Sun, Y., Dong, L., Zhu, Y., Huang, S., Wang, W., Ma, S., Zhang, Q., Wang, J., and Wei, F.
\newblock You only cache once: Decoder-decoder architectures for language models, 2024.
\newblock URL \url{https://arxiv.org/abs/2405.05254}.

\bibitem[Sun et~al.(2025)Sun, Li, Dalal, Xu, Vikram, Zhang, Dubois, Chen, Wang, Koyejo, Hashimoto, and Guestrin]{sun2025learninglearntesttime}
Sun, Y., Li, X., Dalal, K., Xu, J., Vikram, A., Zhang, G., Dubois, Y., Chen, X., Wang, X., Koyejo, S., Hashimoto, T., and Guestrin, C.
\newblock Learning to (learn at test time): Rnns with expressive hidden states, 2025.
\newblock URL \url{https://arxiv.org/abs/2407.04620}.

\bibitem[Sutawika et~al.(2024)Sutawika, Schoelkopf, Gao, Abbasi, Biderman, Tow, ben fattori, Lovering, farzanehnakhaee70, Phang, Thite, Fazz, Aflah, Muennighoff, Wang, sdtblck, nopperl, gakada, tttyuntian, researcher2, Etxaniz, Chris, Lee, Kasner, Khalid, LSinev, Hsu, Kanekar, KonradSzafer, and AndyZwei]{lintang_sutawika_2024_12608602}
Sutawika, L., Schoelkopf, H., Gao, L., Abbasi, B., Biderman, S., Tow, J., ben fattori, Lovering, C., farzanehnakhaee70, Phang, J., Thite, A., Fazz, Aflah, Muennighoff, N., Wang, T., sdtblck, nopperl, gakada, tttyuntian, researcher2, Etxaniz, J., Chris, Lee, H.~A., Kasner, Z., Khalid, LSinev, Hsu, J., Kanekar, A., KonradSzafer, and AndyZwei.
\newblock Eleutherai/lm-evaluation-harness: v0.4.3, July 2024.
\newblock URL \url{https://doi.org/10.5281/zenodo.12608602}.

\bibitem[Tillet et~al.(2019)Tillet, Kung, and Cox]{tillet2019triton}
Tillet, P., Kung, H.-T., and Cox, D.
\newblock Triton: an intermediate language and compiler for tiled neural network computations.
\newblock In \emph{Proceedings of the 3rd ACM SIGPLAN International Workshop on Machine Learning and Programming Languages}, pp.\  10--19, 2019.

\bibitem[Vaswani et~al.(2023)Vaswani, Shazeer, Parmar, Uszkoreit, Jones, Gomez, Kaiser, and Polosukhin]{vaswani2023attentionneed}
Vaswani, A., Shazeer, N., Parmar, N., Uszkoreit, J., Jones, L., Gomez, A.~N., Kaiser, L., and Polosukhin, I.
\newblock Attention is all you need, 2023.
\newblock URL \url{https://arxiv.org/abs/1706.03762}.

\bibitem[Wang et~al.(2020)Wang, Li, Khabsa, Fang, and Ma]{wang2020linformerselfattentionlinearcomplexity}
Wang, S., Li, B.~Z., Khabsa, M., Fang, H., and Ma, H.
\newblock Linformer: Self-attention with linear complexity, 2020.
\newblock URL \url{https://arxiv.org/abs/2006.04768}.

\bibitem[Yang \& Zhang(2024)Yang and Zhang]{yang2024fla}
Yang, S. and Zhang, Y.
\newblock Fla: A triton-based library for hardware-efficient implementations of linear attention mechanism, January 2024.
\newblock URL \url{https://github.com/fla-org/flash-linear-attention}.

\bibitem[Yang et~al.(2024)Yang, Wang, Shen, Panda, and Kim]{yang2024gatedlinearattentiontransformers}
Yang, S., Wang, B., Shen, Y., Panda, R., and Kim, Y.
\newblock Gated linear attention transformers with hardware-efficient training, 2024.
\newblock URL \url{https://arxiv.org/abs/2312.06635}.

\bibitem[Yang et~al.(2025{\natexlab{a}})Yang, Kautz, and Hatamizadeh]{yang2025gateddeltanetworksimproving}
Yang, S., Kautz, J., and Hatamizadeh, A.
\newblock Gated delta networks: Improving mamba2 with delta rule, 2025{\natexlab{a}}.
\newblock URL \url{https://arxiv.org/abs/2412.06464}.

\bibitem[Yang et~al.(2025{\natexlab{b}})Yang, Wang, Zhang, Shen, and Kim]{yang2025parallelizinglineartransformersdelta}
Yang, S., Wang, B., Zhang, Y., Shen, Y., and Kim, Y.
\newblock Parallelizing linear transformers with the delta rule over sequence length, 2025{\natexlab{b}}.
\newblock URL \url{https://arxiv.org/abs/2406.06484}.

\bibitem[Zellers et~al.(2019)Zellers, Holtzman, Bisk, Farhadi, and Choi]{zellers2019hellaswagmachinereallyfinish}
Zellers, R., Holtzman, A., Bisk, Y., Farhadi, A., and Choi, Y.
\newblock Hellaswag: Can a machine really finish your sentence?, 2019.
\newblock URL \url{https://arxiv.org/abs/1905.07830}.

\end{thebibliography}

\newpage
\appendix
\onecolumn

\section{Datasets}~\label{app:dataset}
\vspace{-2em}
\paragraph{Commonsense Reasoning.}
We evaluate on zero-shot commonsense reasoning benchmarks.
For multiple-choice tasks, we report task accuracy on PIQA~\citep{bisk2019piqareasoningphysicalcommonsense}, HellaSwag~\citep{zellers2019hellaswagmachinereallyfinish}, WinoGrande~\citep{sakaguchi2019winograndeadversarialwinogradschema}, ARC-Easy, ARC-Challenge~\citep{clark2018thinksolvedquestionanswering}, OpenBookQA~\citep{mihaylov2018suitarmorconductelectricity}, and BoolQ~\citep{clark-etal-2019-boolq}, as well as language modeling tasks on WikiText~\citep{merity2016pointersentinelmixturemodels} and LAMBADA~\citep{paperno2016lambadadatasetwordprediction}.
All evaluations are conducted in a zero-shot setting using the LM Evaluation Harness~\citep{eval-harness}.

\paragraph{In-Context Retrieval.}
To assess retrieval capacities, we consider both synthetic and real-world in-context retrieval benchmarks.
For synthetic evaluation, we use the Needle-In-A-Haystack Single (NIAH-S) benchmark from RULER~\citep{hsieh2024rulerwhatsrealcontext}, which consists of three tasks, passkey retrieval (S-NIAH-1), numerical needle retrieval (S-NIAH-2), and word-based needle retrieval (S-NIAH-3).
These tasks evaluate a model’s ability to recover sparse target information embedded at arbitrary positions within long contexts.
For real-world retrieval, we follow the evaluation protocol introduced by~\cite{arora2024just}.
These include SWDE~\citep{lockard-etal-2019-openceres} for structured HTML relation extraction, FDA~\citep{arora2023zoologymeasuringimprovingrecall} for key--value retrieval from PDFs, and several question-answering datasets such as SQuAD~\citep{rajpurkar-etal-2018-know}, TriviaQA~\citep{joshi-etal-2017-triviaqa}, DROP~\citep{dua-etal-2019-drop}, and Natural Questions (NQ)~\citep{kwiatkowski-etal-2019-natural}.
All real-world retrieval inputs are truncated to a maximum context length of 2K tokens.
Since our pretrained models are not instruction-tuned, we adopt cloze completion prompts as provided by prior work~\citep{yang2025gateddeltanetworksimproving}.

\section{Query-Feedback Coefficient $\lambda_t$}
\label{app:lambimpl}
The coefficient $\lambda_t$ is computed per head from the hidden state as
$\lambda_t = \sigma(W_\lambda h_t + b)$, where $\sigma$ is the logistic sigmoid,
$W_\lambda \in \mathbb{R}^{H \times d_{\text{model}}}$ with $H$ the number of heads,
and $b$ is a scalar bias initialized to $-0.8$. 
This adds the single projection $W_\lambda$ over the gated delta rule, introducing no additional recurrent state or $q,k,v$ transformation.

\section{Theoretical Derivations}
\subsection{Query for Value Prediction}~\label{app:qvp_proof}
We consider the linear recurrence
\begin{equation}
\label{eq:appendix_recurrence}
S_t = S_{t-1} P_t + \eta_t v_t k_t^\top,
\qquad
S_0 = 0,
\end{equation}
where $S_t \in \mathbb{R}^{d_v \times d_k}$, $P_t \in \mathbb{R}^{d_k \times d_k}$ is a linear operator, $v_t \in \mathbb{R}^{d_v}$, $k_t \in \mathbb{R}^{d_k}$, and $\eta_t \in \mathbb{R}$. 

We show by induction that for all $t \ge 1$, there exist vectors
$\{b_{\tau,t}\}_{\tau \le t} \subset \mathbb{R}^{d_k}$ such that
\begin{equation}
\label{eq:appendix_state_decomp}
S_t = \sum_{\tau=1}^{t} v_\tau b_{\tau,t}^\top.
\end{equation}
For $t=0$, $S_0 = 0$ and the claim holds trivially.
Assume that Eq.~\eqref{eq:appendix_state_decomp} holds for $S_{t-1}$, i.e.,
\[
S_{t-1} = \sum_{\tau=1}^{t-1} v_\tau b_{\tau,t-1}^\top.
\]
Substituting into Eq.~\eqref{eq:appendix_recurrence} gives
\begin{align*}
S_t
&= S_{t-1} P_t + \eta_t v_t k_t^\top \\
&= \left( \sum_{\tau=1}^{t-1} v_\tau b_{\tau,t-1}^\top \right) P_t
   + \eta_t v_t k_t^\top \\
&= \sum_{\tau=1}^{t-1} v_\tau \bigl(b_{\tau,t-1}^\top P_t\bigr)
   + v_t (\eta_t k_t)^\top.
\end{align*}
Using the identity $b^\top P = (P^\top b)^\top$, define
\[
b_{\tau,t} := P_t^\top b_{\tau,t-1} \in \mathbb{R}^{d_k}, \quad \tau = 1,\dots,t-1,
\qquad
b_{t,t} := \eta_t k_t.
\]
Then
\[
S_t = \sum_{\tau=1}^{t} v_\tau b_{\tau,t}^\top,
\]
which completes the induction.

For any query $q_t \in \mathbb{R}^{d_k}$, the query-conditioned prediction satisfies
\begin{align*}
\hat{o}_t := S_{t-1} q_t
&= \sum_{\tau=1}^{t-1} v_\tau (b_{\tau,t-1}^\top q_t) \\
&= \sum_{\tau=1}^{t-1} \gamma_{\tau,t} v_\tau,
\qquad
\gamma_{\tau,t} := b_{\tau,t-1}^\top q_t \in \mathbb{R}.
\end{align*}

\subsection{Stability Analysis of Q-Delta Dynamics}~\label{app:stability_proof}

\onesteplemma*

\begin{proof}
Let $x_t := k_t+\lambda_t q_t$ and consider the Q-Delta update
\[
S_t = S_{t-1} + \beta_t\bigl(v_t - S_{t-1}x_t\bigr)k_t^\top .
\]
Right-multiply both sides by $x_t$ to obtain
\[
S_t x_t
=
S_{t-1}x_t + \beta_t\bigl(v_t - S_{t-1}x_t\bigr)k_t^\top x_t.
\]
Define the scalar alignment $a_t := k_t^\top x_t$. Then
\[
S_t x_t
=
S_{t-1}x_t + \beta_t a_t\bigl(v_t - S_{t-1}x_t\bigr).
\]
Rearranging gives the exact identity
\[
v_t - S_t x_t
=
v_t - S_{t-1}x_t - \beta_t a_t\bigl(v_t - S_{t-1}x_t\bigr)
=
(1-\beta_t a_t)\bigl(v_t - S_{t-1}x_t\bigr).
\]
Taking norms yields
\[
\|v_t - S_t x_t\|
=
|1-\beta_t a_t|\;\|v_t - S_{t-1}x_t\|.
\]
By assumption $\beta_t a_t\in(0,2)$ almost surely, hence $|1-\beta_t a_t|<1$.
With $\rho := \sup_t |1-\beta_t a_t|\in(0,1)$, we therefore have
\[
\|v_t - S_t x_t\|
\le
\rho\,\|v_t - S_{t-1}x_t\|
\qquad\text{almost surely for all }t.
\]
\end{proof}

\globalcontracprop*

\begin{proof}
By definition of the Q-Delta update,
\[
S_t = S_{t-1} + \beta_t\bigl(v_t - S_{t-1}x_t\bigr)k_t^\top.
\]
Right-multiplying by $x_t$ gives
\[
S_t x_t
=
S_{t-1}x_t + \beta_t\bigl(v_t - S_{t-1}x_t\bigr)k_t^\top x_t
=
S_{t-1}x_t + \beta_t a_t\bigl(v_t - S_{t-1}x_t\bigr),
\]
where $a_t := k_t^\top x_t$. Rearranging and using the definitions
\[
\tilde e_t := v_t - S_{t-1}x_t,
\qquad
e_t := v_t - S_t x_t,
\]
yields the exact identity
\[
e_t
=
v_t - S_t x_t
=
v_t - \Bigl(S_{t-1}x_t + \beta_t a_t (v_t - S_{t-1}x_t)\Bigr)
=
(1-\beta_t a_t)\tilde e_t.
\]
Taking norms and using $\rho := \sup_t |1-\beta_t a_t|<1$ gives
\begin{equation}
\label{eq:et_contr_pre}
\|e_t\| \le \rho\,\|\tilde e_t\|\qquad \forall t.
\end{equation}

Next, starting from $\tilde e_t = v_t - S_{t-1}x_t$, add and subtract $v_{t-1}$ and $S_{t-1}x_{t-1}$:
\[
\tilde e_t
=
(v_{t-1} - S_{t-1}x_{t-1})
+
(v_t - v_{t-1})
-
S_{t-1}(x_t - x_{t-1}).
\]
Recognize that $v_{t-1} - S_{t-1}x_{t-1} = e_{t-1}$, and define the target drift
$\Delta_t v := v_t - v_{t-1}$ and the prediction drift
$\Delta_t p := S_{t-1}(x_t - x_{t-1})$ to obtain
\[
\tilde e_t = e_{t-1} + \Delta_t v - \Delta_t p
= e_{t-1} + r_t,
\]
where the residual drift is $r_t := \Delta_t v - \Delta_t p$.
Taking norms and applying the triangle inequality gives
\begin{equation}
\label{eq:pre_bound_rt}
\|\tilde e_t\|
\le
\|e_{t-1}\| + \|r_t\|.
\end{equation}

Combining \eqref{eq:et_contr_pre} and \eqref{eq:pre_bound_rt} yields
\[
\|e_t\|
\le
\rho\bigl(\|e_{t-1}\| + \|r_t\|\bigr).
\]
Under the uniform boundedness assumption, i.e., $\|r_t\|\le  r$ for all $t\ge 1$, we obtain the recurrence
\[
\|e_t\| \le \rho\|e_{t-1}\| + \rho r.
\]
Unrolling this gives
\begin{align*}
\|e_t\|
&\le \rho\|e_{t-1}\|+\rho r \\
&\le \rho\bigl(\rho\|e_{t-2}\|+\rho r\bigr)+\rho r \\
&= \rho^2\|e_{t-2}\|+\rho r(\rho+1) \\
&\ \ \vdots \\
&\le \rho^t\|e_0\| + \rho r\sum_{j=0}^{t-1}\rho^j.
\end{align*}
Solving the geometric series further gives
\[
\|e_t\|
\le
\rho^t\|e_0\| + \frac{1-\rho^t}{1-\rho}\,\rho r.
\]
\end{proof}

\section{Chunkwise Parallelization of Q-Delta}

\subsection{WY representation}\label{app:wy_representation}

Here we drive $P^r$ and $G^r$ in terms of
$\gamma$, $w^i$, and $u^i$.
For clarity, we drop the chunk index $[t]$ and write
$k_i := k_{t_i}$, $x_i := x_{t_i}$, $v_i := v_{t_i}$,
$\alpha_i := \alpha_{t_i}$, $\beta_i := \beta_{t_i}$, $w^i := w_{[t]}^i$, $u^i := u_{[t]}^i$.
Define
\[
P_i := I - \beta_i x_i k_i^\top,
\qquad
\gamma^r := \prod_{j=1}^r \alpha_j,
\quad (\gamma^0 := 1).
\]
Recall that 
$F^r = \prod_{i=1}^r \alpha_i P_i$ in Eq.~\eqref{eq:qdelta_chunk_expand}, which gives
\[
F^r
=
\Bigl(\prod_{i=1}^r \alpha_i\Bigr)\Bigl(\prod_{i=1}^r P_i\Bigr)
=
\gamma^r P^r,
\qquad
P^r := \prod_{i=1}^r P_i .
\]
Assume inductively that
\[
P^{r-1} = I - \sum_{i=1}^{r-1} w^i k_i^\top .
\]
Multiplying by $P_r$ gives
\begin{align}
P^r
&= P^{r-1} P_r \nonumber\\
&= \Bigl(I - \sum_{i=1}^{r-1} w^i k_i^\top\Bigr)
    \Bigl(I - \beta_r x_r k_r^\top\Bigr) \nonumber\\
&= I - \sum_{i=1}^{r-1} w^i k_i^\top
    - \beta_r x_r k_r^\top
    + \beta_r \sum_{i=1}^{r-1} w^i (k_i^\top x_r) k_r^\top \nonumber\\
&= I - \sum_{i=1}^{r-1} w^i k_i^\top
    - \beta_r\Bigl(x_r - \sum_{i=1}^{r-1} w^i (k_i^\top x_r)\Bigr) k_r^\top .
\end{align}
Define
\[
w^r := \beta_r\Bigl(x_r - \sum_{i=1}^{r-1} w^i (k_i^\top x_r)\Bigr),
\]
which then yields
\[
P^r = I - \sum_{i=1}^r w^i k_i^\top .
\]

Recall that the additive term in the chunkwise expansion (Eq.~\eqref{eq:qdelta_chunk_expand}) is
\[
G^r
=
\sum_{i=1}^r \beta_i v_i k_i^\top \prod_{j=i+1}^r \alpha_j P_j .
\]
Splitting the product,
\[
\prod_{j=i+1}^r \alpha_j P_j
=
\Bigl(\prod_{j=i+1}^r \alpha_j\Bigr)
\Bigl(\prod_{j=i+1}^r P_j\Bigr)
=
\frac{\gamma^r}{\gamma^i}\prod_{j=i+1}^r P_j ,
\]
and therefore
\[
G^r
=
\sum_{i=1}^r \frac{\gamma^r}{\gamma^i}\,
\beta_i v_i k_i^\top
\prod_{j=i+1}^r P_j .
\]
From above, $G^r$ satisfies the recursion
\[
G^r = \alpha_r G^{r-1} P_r + \beta_r v_r k_r^\top,
\qquad
G^0 = 0 .
\]
Assume inductively that
\[
G^{r-1} = \sum_{i=1}^{r-1} \frac{\gamma^{r-1}}{\gamma^i} \tilde{u}^i k_i^\top .
\]
Multiplying by $\alpha_r$ gives
\[
\alpha_r G^{r-1}
=
\sum_{i=1}^{r-1} \frac{\gamma^r}{\gamma^i} \tilde{u}^i k_i^\top ,
\]
since $\gamma^r = \alpha_r \gamma^{r-1}$.
Multiplying by $P_r$ and adding the new term yields
\begin{align}
G^r
&=
\Bigl(\sum_{i=1}^{r-1} \frac{\gamma^r}{\gamma^i} \tilde{u}^i k_i^\top\Bigr)
\Bigl(I - \beta_r x_r k_r^\top\Bigr)
+ \beta_r v_r k_r^\top \nonumber\\
&=
\sum_{i=1}^{r-1} \frac{\gamma^r}{\gamma^i} \tilde{u}^i k_i^\top
- \beta_r \sum_{i=1}^{r-1} \frac{\gamma^r}{\gamma^i} \tilde{u}^i (k_i^\top x_r) k_r^\top
+ \beta_r v_r k_r^\top \nonumber\\
&=
\sum_{i=1}^{r-1} \frac{\gamma^r}{\gamma^i} \tilde{u}^i k_i^\top
+ \beta_r\Bigl(v_r - \sum_{i=1}^{r-1} \frac{\gamma^r}{\gamma^i} \tilde{u}^i (k_i^\top x_r)\Bigr) k_r^\top .
\end{align}
To match the desired form
$G^r = \sum_{i=1}^r \frac{\gamma^r}{\gamma^i} \tilde{u}^i k_i^\top$,
we define
\[
u^r := \beta_r\Bigl(v_r - \sum_{i=1}^{r-1} \frac{\gamma^r}{\gamma^i} \tilde{u}^i (k_i^\top x_r)\Bigr).
\]
Therefore, we have following relations,
\[
P^r = I - \sum_{i=1}^r w^i k_i^\top,
\qquad
w^r = \beta_r\Bigl(x_r - \sum_{i=1}^{r-1} w^i (k_i^\top x_r)\Bigr),
\]
\[
G^r = \sum_{i=1}^r \frac{\gamma^r}{\gamma^i} \tilde{u}^i k_i^\top,
\qquad
u^r = \beta_r\Bigl(v_r - \sum_{i=1}^{r-1} \frac{\gamma^r}{\gamma^i} \tilde{u}^i (k_i^\top x_r)\Bigr),
\]
which completes the extended WY representation.

\subsection{UT transform}
\label{app:ut_transform}

The extended WY recursion defining $\tilde u^r$ is
\begin{equation}
\label{eq:tilde_u_rec}
\tilde u^r
=
\beta_r\left(
v_r - \sum_{i=1}^{r-1}\frac{\gamma^{r}}{\gamma^{i}}\tilde u^i (k_i^\top x_r)
\right).
\end{equation}
We now rewrite this in matrix form.
Let
\[
\widetilde U \in \mathbb{R}^{C\times d_v},\quad
V\in\mathbb{R}^{C\times d_v},\quad
K\in\mathbb{R}^{C\times d_k},\quad
X\in\mathbb{R}^{C\times d_k}
\]
stack rows $\tilde u^r$, $v_r$, $k_r$, $x_r$ respectively.
Let $B:=\text{diag}(\beta)\in\mathbb{R}^{C\times C}$.
Define $\Gamma\in\mathbb{R}^{C\times C}$ by
\[
\Gamma_{ri}:=
\begin{cases}
\gamma^{r}/\gamma^{i}, & r>i,\\
0, & r\le i,
\end{cases}
\qquad\text{(strictly lower triangular)}.
\]
Now define
\[
L_\gamma := \text{strictLower}\bigl(B(\Gamma \odot KX^\top)\bigr)\in\mathbb{R}^{C\times C},
\]
equivalently, for each row $r$,
\[
(L_\gamma \widetilde U)_{r,:}
=
\sum_{i<r}\beta_r\Gamma_{ri}(k_i^\top x_r)\tilde u^i{}^\top,
\]

Then we can rewrite Eq.~\eqref{eq:tilde_u_rec} for all $r$ as
\[
\widetilde U + L_\gamma \widetilde U = B V,
\]
hence
\begin{equation}
\label{eq:UT_gamma}
\widetilde U = (I+L_\gamma)^{-1} B V.
\end{equation}
Therefore, defining the UT transform matrix
\[
T_\gamma := (I+L_\gamma)^{-1}B,
\]
we obtain the matrix form
\[
\widetilde U = T_\gamma V.
\]


\end{document}